\documentclass[runningheads]{llncs}
\usepackage[T1]{fontenc}
\usepackage{graphicx}
\usepackage{amssymb}
\usepackage{graphicx}
\usepackage{amsfonts}
\usepackage{amsmath,bm}
\usepackage{bbm}
\usepackage{multirow}
\usepackage{booktabs}
\usepackage{makecell}
\usepackage{color}
\usepackage{array}
\usepackage{algorithm}
\usepackage{enumerate}
\usepackage{algpseudocode}
\usepackage{url}
\usepackage[misc]{ifsym}

\newtheorem{assumption}{Assumption}
\begin{document}
\title{Improved Regret Bounds for Online Kernel Selection under Bandit Feedback
\thanks{This work was accepted by ECML-PKDD 2022.}}
\titlerunning{Improved Regret Bounds for Online Kernel Selection under Bandit Feedback}
%
%
\author{Junfan Li\orcidID{0000-0003-1027-4251} \and Shizhong Liao\orcidID{0000-0003-0594-7116} (\Letter)}
\authorrunning{J.Li and S.Liao}
%
\institute{College of Intelligence and Computing, Tianjin University, \\Tianjin 300350, China\\
\email{\{junfli,szliao\}@tju.edu.cn}}
\toctitle{Improved Regret Bounds for Online Kernel Selection under Bandit Feedback}
\tocauthor{Junfan~Li and Shizhong~Liao}
\maketitle              
\begin{abstract}
  In this paper,
  we improve the regret bound for online kernel selection under bandit feedback.
  Previous algorithm enjoys a $O((\Vert f\Vert^2_{\mathcal{H}_i}+1)K^{\frac{1}{3}}T^{\frac{2}{3}})$
  expected bound for Lipschitz loss functions.
  We prove two types of regret bounds improving the previous bound.
  For smooth loss functions,
  we propose an algorithm with a
  $O(U^{\frac{2}{3}}K^{-\frac{1}{3}}(\sum^K_{i=1}L_T(f^\ast_i))^{\frac{2}{3}})$ expected bound
  where $L_T(f^\ast_i)$ is
  the cumulative losses of optimal hypothesis in $\mathbb{H}_{i}
  =\{f\in\mathcal{H}_i:\Vert f\Vert_{\mathcal{H}_i}\leq U\}$.
  The data-dependent bound keeps the previous worst-case bound
  and is smaller if most of candidate kernels match well with the data.
  For Lipschitz loss functions,
  we propose an algorithm with a $O(U\sqrt{KT}\ln^{\frac{2}{3}}{T})$ expected bound
  asymptotically improving the previous bound.
  We apply the two algorithms to online kernel selection with time constraint
  and prove new regret bounds matching or improving
  the previous $O(\sqrt{T\ln{K}}
  +\Vert f\Vert^2_{\mathcal{H}_i}\max\{\sqrt{T},\frac{T}{\sqrt{\mathcal{R}}}\})$ expected bound
  where $\mathcal{R}$ is the time budget.
  Finally, we empirically verify our algorithms on online regression and classification tasks.
\end{abstract}
\section{Introduction}

    Selecting a suitable kernel function is critical for online kernel learning algorithms,
    and is more challenge than offline kernel selection
    since the data are provided sequentially and may not be i.i.d..
    Such kernel selection problems are named online kernel selection \cite{Zhang2018Online}.
    To address those challenges,
    many online kernel selection algorithms
    reduce it to a sequential decision problem,
    and then randomly select a kernel function
    or use a convex combination of multiple kernel functions on the fly
    \cite{Yang2012Online,Hoi2013Online,Zhang2018Online,Shen2019Random,Liao2021High}.
    Let $\mathcal{K}=\{\kappa_i\}^K_{i=1}$ be predefined base kernels.
    An adversary sequentially sends the learner instances $\{\mathbf{x}_t\}^T_{t=1}$.
    The learner will choose a sequence of hypotheses $\{f_{t}\}^T_{t=1}$
    from the $K$ reproducing kernel Hilbert spaces (RKHSs) induced by kernels in $\mathcal{K}$.
    At each round $t$,
    the learner suffers a prediction loss $\ell(f_t(\mathbf{x}_t),y_t)$.
    The goal is to minimize the regret defined as follows,
    \begin{equation}
    \label{eq:ECML2022:definition_regret}
        \forall \kappa_i\in\mathcal{K},\forall f\in\mathcal{H}_i,\quad
        \mathrm{Reg}_T(f)=\sum^T_{t=1}\ell(f_t(\mathbf{x}_t),y_t)
        -\sum^T_{t=1}\ell(f(\mathbf{x}_t),y_t).
    \end{equation}
    Effective online kernel selection algorithms must keep
    sublinear regret bounds w.r.t. the unknown optimal RKHS $\mathcal{H}_{i^\ast}$
    induced by $\kappa_{i^\ast}\in\mathcal{K}$.

    Previous work reduces online kernel selection to a sequential decision problem,
    including
    (i) prediction with expert advice \cite{Cesa-Bianchi2006Prediction},
    (ii) $K$-armed bandit problem \cite{Auer2002The},
    (iii) prediction with limited advice \cite{Seldin2014Prediction}.
    The online multi-kernel learning algorithms \cite{Sahoo2014Online,Foster2017Parameter}
    which reduce the problem to prediction with expert advice,
    use a convex combination of $K$ hypotheses
    and enjoy a $O(\mathrm{poly}(\Vert f\Vert_{\mathcal{H}_i})\sqrt{T\ln{K}})$ regret bound.
    Combining $K$ hypotheses induces a $O(Kt)$ per-round time complexity
    which is linear with $K$.
    To reduce the time complexity,
    the OKS algorithm (Online Kernel Selection) \cite{Yang2012Online}
    reduces the problem to an adversarial $K$-armed bandit problem.
    OKS randomly selects a hypothesis per-round and only provides a
    $O(\mathrm{poly}(\Vert f\Vert_{\mathcal{H}_i})K^{\frac{1}{3}}T^{\frac{2}{3}}))$
    \footnote{$\mathrm{poly}(\Vert f\Vert_{\mathcal{H}_i})=\Vert f\Vert^2_{\mathcal{H}_i}+1$.
    The original paper shows a $O((\Vert f\Vert^2_{\mathcal{H}_i}+1)\sqrt{KT})$
    expected regret bound.
    We will clarify the difference in Section \ref{sec:ECML2022:problem_setting}.}
    expected bound.
    The per-round time complexity of OKS is $O(t)$.
    The $\mathrm{B(AO)_2KS}$ algorithm \cite{Liao2021High}
    reduces the problem to predict with limited advice
    and randomly selects two hypotheses per-round.
    $\mathrm{B(AO)_2KS}$
    can provide a $\tilde{O}(\mathrm{poly}(\Vert f\Vert_{\mathcal{H}_i})\sqrt{KT})$ high-probability bound
    and suffers a $O(t/K)$ per-round time complexity.
    From the perspective of algorithm design,
    an important question arises:
    does there exist some algorithm only selecting a hypothesis
    (or under bandit feedback)
    improving the
    $O(\mathrm{poly}(\Vert f\Vert_{\mathcal{H}_i})K^{\frac{1}{3}}T^{\frac{2}{3}}))$
    expected bound?
    The significances of answering the question include
    (i) explaining the information-theoretic cost induced by only selecting a hypothesis
    (or observing a loss);
    (ii) designing better algorithms for online kernel selection with time constraint.
    In this paper,
    we will answer the question affirmatively.

    We consider Lipschitz loss functions and
    smooth loss functions (Assumption \ref{ass:AAAI2022:smoothness}).
    For Lipschitz loss functions,
    we propose an algorithm
    whose expected regret bound is $O(U\sqrt{KT}\ln^{\frac{2}{3}}{T})$ asymptotically improving the
    $O(\mathrm{poly}(\Vert f\Vert_{\mathcal{H}_i})K^{\frac{1}{3}}T^{\frac{2}{3}}))$
    expected bound.
    Our regret bound proves that selecting a or multiple hypotheses
    will not induce significant variation on the worst-case regret bound.
    For smooth loss functions,
    we propose an adaptive parameter tuning scheme for OKS and prove a
    $O(U^{\frac{2}{3}}K^{-\frac{1}{3}}(\sum^K_{j=1}L_T(f^\ast_j))^{\frac{2}{3}})$
    expected bound
    where
    $L_T(f^\ast_j)=\min_{f\in\mathbb{H}_j}\sum_{t\in[T]}\ell(f(\mathbf{x}_t),y_t)$.
    If most of base kernels in $\mathcal{K}$ match well with the data,
    i.e., $L_T(f^\ast_j)\ll T$,
    then the data-dependent regret bound significantly
    improves the previous worst-case bound.
    In the worst case,
    i.e., $L_T(f^\ast_j)=O(T)$,
    the data-dependent bound is still same with the previous bound.
    Our new regret bounds answer the above question.
    We summary the results in Table \ref{tab:ECML2022:summary_main_result}.

    \begin{table}[!t]
      \centering
      \label{tab:ECML2022:summary_main_result}
      \caption{Expected regret bounds for online kernel selection under bandit feedback.
      $\mathcal{R}$ is the time budget.
      $\bar{L}_{T}=\sum^K_{j=1}L_T(f^\ast_j)$.
      $\nu$ is a parameter in the definition of smooth loss
      (see Assumption \ref{ass:AAAI2022:smoothness}).
      There is no algorithm under bandit feedback in the case of a time budget.
      Thus we report the result produced under the expert advice model \cite{Li2022Worst}.}
      \begin{tabular}{l|l|r|r}
      \toprule
        $\mathcal{R}$&Loss function&Previous results&Our results\\
      \toprule
      \multirow{3}{*}{No}&
        Lipschitz loss  &
        \multirow{3}{*}{$O\left(\mathrm{poly}(\Vert f\Vert_{\mathcal{H}_i})K^{\frac{1}{3}}T^{\frac{2}{3}}\right)$
        \cite{Yang2012Online}}
        & {\color{blue}$O(U\sqrt{KT}\ln^{\frac{2}{3}}{T})$}  \\
        &Smooth loss $\nu=1$  &   & {\color{blue}$O(U^{\frac{2}{3}}K^{-\frac{1}{3}}\bar{L}^{\frac{2}{3}}_{T})$}  \\
        &Smooth loss $\nu=2$   &   & {\color{blue}$O(U^{\frac{2}{3}}K^{-\frac{1}{3}}\bar{L}^{\frac{2}{3}}_{T})$}  \\
      \toprule
      \multirow{3}{*}{Yes}&
        Lipschitz loss    &
        \multirow{3}{*}{$O\left(\Vert f\Vert^2_{\mathcal{H}_i}
        \max\{\sqrt{T},\frac{T}{\sqrt{\mathcal{R}}}\}\right)$
        \cite{Li2022Worst}}
        & {\color{blue}$O(U\sqrt{KT}\ln^{\frac{2}{3}}{T}+\frac{UT\sqrt{\ln{T}}}{\sqrt{\mathcal{R}}})$}  \\
        &Smooth loss $\nu=1$   &    & {\color{blue}$\tilde{O}(U^{\frac{2}{3}}K^{-\frac{1}{3}}\bar{L}^{\frac{2}{3}}_{T}
        +\frac{UL_T(f^\ast_i)}{\sqrt{\mathcal{R}}})$}  \\
        &Smooth loss $\nu=2$   &    & {\color{blue}
        $\tilde{O}(U^{\frac{2}{3}}K^{-\frac{1}{3}}\bar{L}^{\frac{2}{3}}_{T}
        +\frac{U\sqrt{TL_T(f^\ast_i)}}{\sqrt{\mathcal{R}}})$}  \\
      \bottomrule
      \end{tabular}
    \end{table}

    We apply the two algorithms to online kernel selection with time constraint
    where the time of kernel selection and online prediction
    is limited to $\mathcal{R}$ quanta \cite{Li2022Worst}.
    It was proved that any budgeted algorithm must suffer
    an expected regret of order $\Omega(\Vert f^\ast_i\Vert_{\mathcal{H}_i}
    \max\{\sqrt{T},\frac{T}{\sqrt{\mathcal{R}}}\})$
    and the LKMBooks algorithm enjoys a
    $O(\sqrt{T\ln{K}}+\Vert f\Vert^{2}_{\mathcal{H}_i}
    \max\{\sqrt{T},\frac{T}{\sqrt{\mathcal{R}}}\})$
    expected bound \cite{Li2022Worst}.
    LKMBooks uses convex combination to aggregate $K$ hypotheses.
    Raker uses random features to approximate kernel functions
    and also aggregates $K$ hypotheses \cite{Shen2019Random}.
    Raker enjoys a $\tilde{O}((\sqrt{\ln{K}}+\Vert f\Vert^{2}_1)\sqrt{T}
        +\Vert f\Vert_1\frac{T}{\sqrt{\mathcal{R}}})$ bound
    where $f=\sum^T_{t=1}\alpha_t\kappa_i(\mathbf{x}_t,\cdot)$
    and $\Vert f\Vert_1=\Vert\bm{\alpha}\Vert_1$ \cite{Shen2019Random}.
    The two algorithms reduce the problem to prediction with expert advice,
    while our algorithms just use bandit feedback.

    We also use random features and make a mild assumption that
    reduces the time budget $\mathcal{R}$ to the number of features.
    For smooth loss functions,
    we prove two data-dependent regret bounds
    which can improve the previous worst-case bounds \cite{Shen2019Random,Li2022Worst}
    if there is a good kernel in $\mathcal{K}$ that matches well with the data.
    For Lipschitz loss functions,
    our algorithm enjoys a similar upper bound with LKMBooks.
    We also summary the results in Table \ref{tab:ECML2022:summary_main_result}.

\section{Problem Setting}
\label{sec:ECML2022:problem_setting}

    Denote by $\{({\bf x}_t,y_t)\}_{t\in[T]}$ a sequence of examples,
    where ${\bf x}_t \in\mathcal{X}\subseteq\mathbb{R}^d, y\in [-1,1]$
    and $[T] = \{1,2,\ldots,T\}$.
    Let $\kappa(\cdot,\cdot):\mathcal{X} \times \mathcal{X} \rightarrow \mathbb{R}$
    be a positive definite kernel and
    $\mathcal{K} = \{\kappa_{1},\ldots,\kappa_{K}\}$.
    For each $\kappa_i\in \mathcal{K}$,
    let $\mathcal{H}_i= \{f \vert f : \mathcal{X} \rightarrow \mathbb{R}\}$
    be the associated RKHS satisfying
    $\langle f,\kappa_i(\mathbf{x},\cdot)\rangle_{\mathcal{H}_i}=f(\mathbf{x})$,
    $\forall f\in\mathcal{H}_i$.
    Let $\Vert f\Vert^2_{\mathcal{H}_i}=\langle f,f\rangle_{\mathcal{H}_i}$.
    We assume that $\kappa_i(\mathbf{x},\mathbf{x})\leq 1$, $\forall \kappa_i\in\mathcal{K}$.
    Let $\ell(\cdot,\cdot):\mathbb{R}\times \mathbb{R}\rightarrow \mathbb{R}$
    be the loss function.

\subsection{Online Kernel Selection under Bandit Feedback}

    We formulate online kernel selection as a sequential decision problem.
    At any round $t\in[T]$,
    an adversary gives an instance ${\bf x}_t$.
    The learner maintains $K$ hypotheses $\{f_{t,i}\in\mathcal{H}_i\}^K_{i=1}$
    and selects $f_{t}\in \overline{\mathrm{span}(f_{t,i}:i\in[K])}$,
    and outputs $f_{t}(\mathbf{x}_t)$.
    Then the adversary gives $y_t$.
    The learner suffers a prediction loss $\ell(f_t(\mathbf{x}_t),y_t)$.
    The learner aims to minimize the regret w.r.t.
    any $f\in\cup^K_{i=1}\mathcal{H}_i$
    which is defined in \eqref{eq:ECML2022:definition_regret}.
    If the learner only computes a loss $\ell(f_{t,I_t}(\mathbf{x}_t),y_t),I_t\in[K]$,
    then we call it \textit{bandit feedback setting}.
    The learner can also compute $N\in\{2,\ldots,K\}$ losses,
    i.e., $\{\ell(f_{t,{i_j}}(\mathbf{x}_t),y_t)\}^N_{j=1}$, $i_j\in[K]$.
    The OKS algorithm \cite{Yang2012Online} follows the bandit feedback setting.
    The online multi-kernel learning algorithms \cite{Sahoo2014Online,Foster2017Parameter,Shen2019Random}
    correspond to $N=K$.
    The $\mathrm{B(AO)_2KS}$ algorithm \cite{Liao2021High} corresponds to $N=2$.
    From the perspective of computation,
    the per-round time complexity of computing $N$ losses
    is $N$ times larger than the bandit feedback setting.
    From the perspective of regret bound,
    we aim to reveal the information-theoretic cost
    induced by observing a loss (or bandit feedback) not multiple losses (or $N\geq 2$).

\subsection{Regret bound of OKS}
\label{sec:ECML2022:analysis_of_OKS}

    We first prove that the regret bound of OKS \cite{Yang2012Online}
    is $O((\Vert f\Vert^2_{\mathcal{H}_i}+1)K^{\frac{1}{3}}T^{\frac{2}{3}})$,
    and then explain the technical weakness of OKS.

    \begin{algorithm}[!t]
        \caption{OKS}
        \label{alg:IJCAI2022:OKS}
        \begin{algorithmic}[1]
        \Require {$\mathcal{K}= \{\kappa_{1},\ldots,\kappa_{K}\}$, $\delta\in(0,1)$, $\eta$, $\lambda$}
        \Ensure  $\{f_{1,i}=0,w_{1,i}=1\}^K_{i=1}$, $\mathbf{p}_1=\frac{1}{K}\mathbf{1}_K$\;
        \For{t=1,\ldots,T}
            \State Receive ${\bf x}_t$;
            \State Sample a kernel $\kappa_{I_t}$ where $I_t\sim \mathbf{p}_t$;
            \State Update $w_{t+1,I_t}=w_{t,I_t}\exp(-\eta\frac{\ell(f_{t,I_t}(\mathbf{x}_t),y_t)}{p_{t,I_t}})$;
            \State Update $f_{t+1,I_t}=f_{t,I_t}-\lambda
            \frac{\nabla_{f_{t,I_t}}\ell(f_{t,I_t}(\mathbf{x}_t),y_t)}{p_{t,I_t}}$;
            \State Update $\mathbf{q}_{t+1}=\frac{\mathbf{w}_{t+1}}{\sum^K_{j=1}w_{t+1,j}}$
            and set $\mathbf{p}_{t+1}=(1-\delta)\mathbf{q}_{t+1}+\frac{\delta}{K}\mathbf{1}_K$;
        \EndFor
        \State Output: $\mathbf{q}_T$.
        \end{algorithmic}
    \end{algorithm}

    The pseudo-code of OKS is shown Algorithm \ref{alg:IJCAI2022:OKS}.
    Let $\Delta_{K}$ be the $(K-1)$-dimensional simplex.
    At any round $t$,
    OKS maintains $\mathbf{p}_t,\mathbf{q}_t\in\Delta_{K}$.
    OKS samples $f_{t,I_t}$ where $I_t\sim \mathbf{p}_t$,
    and outputs $f_{t,I_t}(\mathbf{x}_t)$.
    For simplicity, we define two notations,
    \begin{align*}
        L_T(f):=\sum^T_{t=1}\ell(f(\mathbf{x}_t),y_t),\quad
        \bar{L}_{\mathbf{q}_{1:T}}:=&\sum^T_{t=1}\sum^K_{i=1}q_{t,i}\ell(f_{t,i}(\mathbf{x}_t),y_t).
    \end{align*}

    \begin{theorem}[\cite{Yang2012Online}]
    \label{thm:IJCAI2022:OKS}
        Assuming that $\ell(f_{t,i}(\mathbf{x}),y)\in [0,\ell_{\max}]$,
        $\forall i\in[K]$, $t\in[T]$,
        and $\Vert\nabla_f\ell(f(\mathbf{x}),y)\Vert_{\mathcal{H}_i}\leq G$, $\forall f\in\mathcal{H}_i$.
        The expected regret of OKS satisfies
        $$
            \forall i\in[K],f\in\mathcal{H}_i,\quad \mathbb{E}\left[\bar{L}_{\mathbf{q}_{1:T}}\right]
            \leq L_T(f)
            +\frac{\Vert f\Vert^2_{\mathcal{H}_i}}{2\lambda}
            +\frac{\lambda KTG^2}{2\delta}+\frac{\eta KT\ell_{\max}^2}{2(1-\delta)}+\frac{\ln{K}}{\eta}.
        $$
        In particular, let $\delta\in(0,1)$ be a constant and $\eta,\lambda=\Theta((KT)^{-\frac{1}{2}})$,
        then the expected regret bound is $O((\Vert f\Vert^2_{\mathcal{H}_i}+1)\sqrt{KT})$.
    \end{theorem}

    \begin{remark}
    \label{remark:ECML2022:parameter_setting_OKS}
        Since $I_t\sim \mathbf{p}_t$,
        the expected cumulative losses of OKS should be
        $\mathbb{E}\left[\bar{L}_{\mathbf{p}_{1:T}}\right]$
        which is different from $\mathbb{E}\left[\bar{L}_{\mathbf{q}_{1:T}}\right]$
        as stated in Theorem \ref{thm:IJCAI2022:OKS}.
        Since $\mathbf{p}_t=(1-\delta)\mathbf{q}_t+\frac{\delta}{K}\mathbf{1}_K$,
        the expected regret of OKS should be redefined as follows
        \begin{align*}
            \forall i\in[K],f\in\mathcal{H}_i,
            ~&\mathbb{E}\left[\bar{L}_{\mathbf{p}_{1:T}}\right]-L_T(f)\\
            \leq&\delta\mathbb{E}\left[\bar{L}_{\frac{1}{K}\mathbf{1}}\right]-\delta L_T(f)
            +\frac{\Vert f\Vert^2_{\mathcal{H}_i}}{2\lambda}
            +\frac{\lambda KTG^2}{2\delta}+\frac{\eta KT\ell^2_{\max}}{2(1-\delta)}+\frac{\ln{K}}{\eta}\\
            \leq&\delta T\ell_{\max}+\frac{\Vert f\Vert^2_{\mathcal{H}_i}}{2\lambda}
            +\frac{\lambda KTG^2}{2\delta}+\frac{\eta KT\ell^2_{\max}}{2(1-\delta)}+\frac{\ln{K}}{\eta}.
        \end{align*}
        To minimize the upper bound,
        let $\delta=(G/\ell_{\max})^{\frac{2}{3}}K^{\frac{1}{3}}T^{-\frac{1}{3}}$,
        $\lambda=\sqrt{\delta/(KTG^2)}$
        and $\eta=\sqrt{2(1-\delta)\ln{K}}/\sqrt{KT\ell^2_{\max}}$.
        The upper bound is
        $O((\Vert f\Vert^2_{\mathcal{H}_i}+1)K^{\frac{1}{3}}T^{\frac{2}{3}})$.
    \end{remark}

    \begin{remark}
        OKS is essentially an offline kernel selection algorithm,
        since it aims to output a hypothesis following $\mathbf{q}_T$ for test datasets
        (see line 8 in Algorithm \ref{alg:IJCAI2022:OKS}).
        Thus Theorem \ref{thm:IJCAI2022:OKS} defines the expected regret using $\{\mathbf{q}_1,\ldots,\mathbf{q}_T\}$,
        and the $O((\Vert f\Vert^2_{\mathcal{H}_i}+1)\sqrt{KT})$ bound is reasonable.
        For online kernel selection,
        we focus on the online prediction performance.
        Since OKS selects $f_{t,I_t}$ following $\mathbf{p}_t$,
        the expected regret should be defined using $\{\mathbf{p}_1,\ldots,\mathbf{p}_T\}$.
    \end{remark}

    We find that the dependence on
    $O(K^{\frac{1}{3}}T^{\frac{2}{3}})$
    comes from the term $\frac{\lambda KTG^2}{2\delta}$
    which upper bounds the cumulative variance of gradient estimators, i.e,
    $$
        \frac{\lambda}{2}\mathbb{E}\left[\sum^T_{t=1}\Vert \tilde{\nabla}_{t,i}\Vert^2_{\mathcal{H}_i}\right]
        \leq \frac{\lambda KTG^2}{2\delta},
        \tilde{\nabla}_{t,i}=\frac{\nabla_{t,i}}{p_{t,i}}
        \mathbb{I}_{i=I_t},\nabla_{t,i}=\nabla_{f_{t,i}}\ell(f_{t,i}(\mathbf{x}_t),y_t).
    $$
    Next we give a simple analysis.
    To start with,
    it can be verified that
    $$
        \mathbb{E}\left[\Vert\tilde{\nabla}_{t,i}\Vert^2_{\mathcal{H}_i}\right]
        =\mathbb{E}\left[p_{t,i}\frac{\Vert\nabla_{t,i}\Vert^2_{\mathcal{H}_i}}{p^2_{t,i}}
        +(1-p_{t,i})\cdot 0\right]
        \leq\mathbb{E}\left[\max_{t=1,\ldots,T}\left(\frac{1}{p_{t,i}}\right)
        \Vert\nabla_{t,i}\Vert^2_{\mathcal{H}_i}\right].
    $$
    Recalling that $p_{t,i}\geq \frac{\delta}{K}$, $\forall i\in[K]$, $t\in[T]$.
    Summing over $t=1,\ldots,T$ yields
    $$
        \sum^T_{t=1}\mathbb{E}\left[\Vert\tilde{\nabla}_{t,i}\Vert^2_{\mathcal{H}_i}\right]
        \leq \frac{K}{\delta}\sum^T_{t=1}\mathbb{E}\left[\Vert\nabla_{t,i}\Vert^2_{\mathcal{H}_i}\right]
        \leq \frac{KTG^2}{\delta}.
    $$
    The regret bound of online gradient descent (this can be found in our supplementary materials)
    depends on $\frac{\lambda}{2}\mathbb{E}\left[\sum^T_{t=1}\Vert \tilde{\nabla}_{t,i}\Vert^2_{\mathcal{H}_i}\right]\leq\frac{\lambda KTG^2}{2\delta}$.
    Thus it is the high variance of $\tilde{\nabla}_{t,i}$
    that causes the $O(K^{\frac{1}{3}}T^{\frac{2}{3}})$ regret bound.

    OKS selects a hypothesis per-round,
    reduces the time complexity to $O(t)$ but damages the regret bound.
    It was proved selecting two hypotheses can improve the regret bound to
    $\tilde{O}((\Vert f\Vert^2_{\mathcal{H}_i}+1)\sqrt{KT})$ \cite{Liao2021High}.
    A natural question arises:
    will selecting a hypothesis induce worse regret bound than selecting two hypotheses?
    From the perspective of algorithm design,
    we concentrate on the question:
    \begin{itemize}
      \item does there exist some algorithm selecting a hypothesis (or under bandit feedback)
            that can improve the
            $O((\Vert f\Vert^2_{\mathcal{H}_i}+1)K^{\frac{1}{3}}T^{\frac{2}{3}}))$ bound?
    \end{itemize}

\section{Improved Regret bounds for Smooth Loss Functions}

    In this section,
    we propose the OKS++ algorithm
    using an adaptive parameter tuning scheme for OKS.
    Specifically,
    we reset the value of $\delta,\eta$ and $\lambda$ in Theorem \ref{thm:IJCAI2022:OKS}
    and prove data-dependent regret bounds for smooth loss functions.
    Such regret bounds can improve the previous worst-case bound
    if most of candidate kernel functions match well with the data.
    Although OKS++ just resets the value of parameters,
    deriving the new regret bounds requires novel and non-trivial analysis.
    To start with,
    we define the smooth loss functions.

    \begin{assumption}[Smoothness condition]
    \label{ass:AAAI2022:smoothness}
        $\ell(\cdot,\cdot)$ is convex w.r.t. the first parameter.
        Denote by $\ell'(a,b)
        =\frac{\mathrm{d}\,\ell(a,b)}{\mathrm{d}\,a}$.
        For any $f(\mathbf{x})$ and $y$,
        there is a constant $C_0>0$ such that
        $$
            \vert\ell'(f(\mathbf{x}),y)\vert^{\nu} \leq  C_0\ell(f(\mathbf{x}),y),\quad \nu\in\{1,2\}.
        $$
    \end{assumption}

    Zhang et al. \cite{Zhang2013Online} considered
    online kernel learning under smooth loss functions with $\nu=1$.
    The logistic loss $\ell(f(\mathbf{x}),y)=\ln(1+\exp(-yf(\mathbf{x})))$
    satisfies Assumption \ref{ass:AAAI2022:smoothness} with $\nu=1$ and $C_0=1$.
    The square loss $\ell(f(\mathbf{x}),y)=(f(\mathbf{x})-y)^2$
    and the squared hinge loss $\ell(f(\mathbf{x}),y)=(\max\{0,1-yf(\mathbf{x})\})^2$
    satisfy Assumption \ref{ass:AAAI2022:smoothness}
    with $\nu=2$ and $C_0=4$.

    Let $U>0$ be a constant.
    We define $K$ restricted hypothesis spaces.
    $\forall i\in[K]$,
    let $\mathbb{H}_i=\{f\in\mathcal{H}_i:\Vert f\Vert_{\mathcal{H}_i}\leq U\}$.
    Then it is natural to derive Assumption \ref{ass:AAAI2022:lipschitz_condition}.

    \begin{assumption}
    \label{ass:AAAI2022:lipschitz_condition}
        $\forall \kappa_i\in\mathcal{K}$ and $\forall f\in\mathbb{H}_i$,
        there exists a constant $G>0$ such that
        $\max_{t\in[T]}\vert\ell'(f(\mathbf{x}_t),y_t)\vert \leq G$.
    \end{assumption}

    It can be verified that many loss functions satisfy the assumption
    and $G$ may depend on $U$.
    For instance,
    if $\ell$ is the square loss,
    then $G\leq 2(U+1)$.
    For simplicity,
    denote by
    $c_{t,i}=\ell(f_{t,i}(\mathbf{x}_t),y_t)$ for all $i\in[K]$ and $t\in[T]$.
    It can be verified that $\max_{t,i}c_{t,i}$ is bounded and depends on $U$.
    Then our algorithm updates $\mathbf{q}_t$ using $c_t$
    (see line 4 and line 6 in Algorithm \ref{alg:IJCAI2022:OKS}).
    Since we use restricted hypothesis spaces,
    our algorithm changes line 5 in Algorithm \ref{alg:IJCAI2022:OKS} as follows
    \begin{equation}
    \label{eq:IJCAI2022:IOKS:type_2_updating}
        f_{t+1,I_t}=\mathop{\arg\min}_{f\in\mathbb{H}_{I_t}}
        \left\Vert f- \left(f_{t,I_t}-\lambda_{t,I_t}
            \frac{\nabla_{f_{t,I_t}}\ell(f_{t,I_t}(\mathbf{x}_t),y_t)}{p_{t,I_t}}\right)
            \right\Vert^2_{\mathcal{H}_{I_t}}.
    \end{equation}
    Except for $\{\lambda_{t,i}\}^K_{i=1}$,
    our algorithm also uses time-variant $\delta_t$ and $\eta_t$.
    We omit the pseudo-code of OKS++ since it is similar with
    Algorithm \ref{alg:IJCAI2022:OKS}.

    Next we show the regret bound.
    For simplicity,
    let $\tilde{C}_{t,K}=\sum^t_{\tau=1}\sum^K_{i=1}\tilde{c}_{\tau,i}$
    where $\tilde{c}_{\tau,i}=\frac{c_{\tau,i}}{p_{\tau,i}}\mathbb{I}_{I_{\tau}=i}$,
    and $\bar{L}_T=\sum^K_{j=1}L_T(f^\ast_j)$
    where $L_T(f^\ast_j)=\min_{f\in\mathbb{H}_j}L_T(f)$.

    \begin{theorem}
    \label{thm:ECML2022:OKS:loss_bound}
        Let $\ell$ satisfy
        Assumption \ref{ass:AAAI2022:smoothness} with $\nu=1$
        and Assumption \ref{ass:AAAI2022:lipschitz_condition}.
        Let
        \begin{align*}
        &\delta_t=\frac{(GC_0)^{\frac{1}{3}}(UK)^{\frac{2}{3}}}
        {2\max\left\{(GC_0)^{\frac{1}{3}}(UK)^{\frac{2}{3}},
        2\tilde{C}^{\frac{1}{3}}_{t,K}\right\}},
        \eta_t=\frac{\sqrt{2\ln{K}}}{\sqrt{1+\sum^t_{\tau=1}\sum^K_{i=1}q_{\tau,i}\tilde{c}^2_{\tau,i}}},\\
        \forall i\in[K],&\lambda_{t,i}= \frac{U^{\frac{4}{3}}(\max\{GC_0U^2K^2,8\tilde{C}_{t,K}\})^{-\frac{1}{6}}}
        {\sqrt{4/3}K^{\frac{1}{6}}(GC_0)^{\frac{1}{3}}\sqrt{1+\Delta_{t,i}}},
        \Delta_{t,i}=\sum^t_{\tau=1}
        \frac{\ell(f_{\tau,i}(\mathbf{x}_{\tau}),y_{\tau})}{p_{\tau,i}}\mathbb{I}_{I_{\tau}=i}.
        \end{align*}
        Then the expected regret of OKS++ satisfies, $\forall i\in[K]$,
        \begin{align*}
            \mathbb{E}\left[\bar{L}_{\mathbf{p}_{1:T}}\right]- L_T(f^\ast_i)
            =O\left(U^{\frac{2}{3}}\left(GC_0\right)^{\frac{1}{3}}K^{-\frac{1}{3}}
            \bar{L}^{\frac{2}{3}}_{T}
            +U^{\frac{2}{3}}(GC_0)^{\frac{1}{3}}K^{\frac{1}{6}}\bar{L}^{\frac{1}{6}}_T
            L^{\frac{1}{2}}_T(f^\ast_i)\right).
        \end{align*}
        Let $\ell$ satisfy
        Assumption \ref{ass:AAAI2022:smoothness} with $\nu=2$.
        Let $G=1$ in $\delta_t$ and $\lambda_{t,i}$.
        $\eta_t$ keeps unchanged.
        Then the expected regret of OKS++ satisfies
        $$
            \forall i\in[K],~\mathbb{E}\left[\bar{L}_{\mathbf{p}_{1:T}}\right]- L_T(f^\ast_i)
            =O\left(U^{\frac{2}{3}}C^{\frac{1}{3}}_0K^{-\frac{1}{3}}\bar{L}^{\frac{2}{3}}_T
            +U^{\frac{2}{3}}C^{\frac{1}{3}}_0K^{\frac{1}{6}}\bar{L}^{\frac{1}{6}}_T
            L^{\frac{1}{2}}_T(f^\ast_i)\right).
        $$
    \end{theorem}

    The values of $\lambda_{t,i}$, $\delta_t$ and $\eta_t$ which depend on the observed losses,
    are important to obtain the data-dependent bounds.
    Beside, it is necessary to set different $\lambda_{t,i}$ for each $i\in[K]$.
    OKS sets a same $\lambda$.
    Thus the changes on the values of $\delta$, $\eta$ and $\lambda$ are non-trivial.
    Our analysis is also non-trivial.
    OKS++ sets time-variant parameters
    and does not require prior knowledge of the nature of the data.

    Now we compare our results with the regret bound in Theorem \ref{thm:IJCAI2022:OKS}.
    The main difference is that we replace $KT$ with a data-dependent complexity
    $\bar{L}_T$.
    In the worst case,
    $\bar{L}_T=O(KT)$
    and our regret bound is $O(K^{\frac{1}{3}}T^{\frac{2}{3}})$
    which is same with the result in Theorem \ref{thm:IJCAI2022:OKS}.
    In some benign environments,
    we expect that $\bar{L}_T\ll KT$ and our regret bound would be smaller.
    For instance, if $L_T(f^\ast_i)=o(T)$ for all $i\in[K]$,
    then our regret is $o(T^{\frac{2}{3}})$
    improving the result in Theorem \ref{thm:IJCAI2022:OKS}.
    If there are only $M<K$ hypothesis spaces such that $L_T(f^\ast_i)=O(T)$,
    where $M$ is independent of $K$,
    then our regret bound is $O((MT)^{\frac{2}{3}}K^{-\frac{1}{3}})$.
    Such a result still improves the dependence on $K$.
    A more interesting result is that,
    if $L_T(f^\ast_i)=O(T^{\frac{3}{4}})$ for all $i\in[K]$,
    then OKS++ achieves a $O(K^{\frac{1}{3}}\sqrt{T})$ regret bound
    which is better than the
    $\tilde{O}(\mathrm{poly}(\Vert f\Vert_{\mathcal{H}_i})\sqrt{KT})$ bound
    achieved by $\mathrm{B(AO)_2KS}$ \cite{Liao2021High}.

\section{Improved Regret bound for Lipschitz Loss Functions}

    In this section,
    we consider Lipschitz loss functions
    and propose a new algorithm with improved worst-case regret bound.

\subsection{Algorithm}

    For the sake of clarity,
    we decompose OKS into two levels.
    At the outer level,
    it uses a procedure similar with Exp3 \cite{Auer2002The} to
    update $\mathbf{p}_t$ and $\mathbf{q}_t$.
    At the inner level,
    it updates $f_{t,I_t}$ using online gradient descent.
    Exp3 can be derived from online mirror descent framework with negative entropy regularizer \cite{Agarwal2017Corralling}, i.e.,
    \begin{equation}
    \label{eq:ECML2022:OMD_Log_barrier}
        \nabla_{\mathbf{q}'_{t+1}}\psi_t(\mathbf{q}'_{t+1})
        =\nabla_{\mathbf{q}_{t}}\psi_t(\mathbf{q}_{t})-\tilde{c}_t,\qquad
        \mathbf{q}_{t+1}=\mathop{\arg\min}_{\mathbf{q}\in\Delta_{K}}
        \mathcal{D}_{\psi_t}(\mathbf{q},\mathbf{q}'_{t+1}),
    \end{equation}
    where $\psi_t(\mathbf{p})=\sum^K_{i=1}\frac{1}{\eta}p_i\ln{p_i}$ is the negative entropy
    and
    $\mathcal{D}_{\psi_t}(\mathbf{p},\mathbf{q})
    =\psi_t(\mathbf{p})-\psi_t(\mathbf{q})-\langle\nabla \psi_t(\mathbf{q}),\mathbf{p}-\mathbf{q}\rangle$
    is the Bregman divergence.
    Different regularizer yields different algorithm.
    We will use
    $\psi_t(\mathbf{p})=\sum^K_{i=1}\frac{-\alpha}{\eta_{t,i}}p^{\frac{1}{\alpha}}_{i}$, $\alpha>1$,
    which slightly modifies the $\alpha$-Tsallis entropy \cite{Tsallis1988Possible,Zimmert2019An}.
    We also use the increasing learning rate scheme in \cite{Agarwal2017Corralling},
    that is $\eta_{t,i}$ is increasing.
    The reason is that if $\eta_{t,i}$ is increasing,
    then there will be a negative term in the regret bound
    which can be used to control the large variance of gradient estimator,
    i.e., $\mathbb{E}\left[\sum^T_{t=1}\Vert\tilde{\nabla}_{t,i}\Vert^2_{\mathcal{H}_i}\right]$
    (see Section \ref{sec:ECML2022:analysis_of_OKS}).
    If we use the log-barrier \cite{Agarwal2017Corralling}
    or $\alpha$-Tsallis entropy with $\alpha=2$ \cite{Audibert2009Minimax,Zimmert2019An},
    then the regret bound will increase a $O(\ln{T})$ factor.
    This factor can be reduced to $O(\ln^{\frac{2}{3}}{T})$ for $\alpha\geq 3$.
    We choose $\alpha=8$ for achieving a small regret bound.

    At the beginning of round $t$,
    our algorithm first samples $I_t \sim\mathbf{p}_t$ and outputs the prediction
    $f_{t,I_t}(\mathbf{x}_t)$ or $\mathrm{sign}(f_{t,I_t}(\mathbf{x}_t))$.
    Next our algorithm updates $f_{t,I_t}$ following \eqref{eq:IJCAI2022:IOKS:type_2_updating}.
    $\forall i\in[K]$, let $c_{t,i}=\ell(f_{t,i}(\mathbf{x}_t),y_t)/\ell_{\max}\in[0,1]$.
    We redefine $\tilde{c}_t$ by
    \begin{equation}
    \label{eq:ECML2022:tilde_c_t:IOKS}
        \mathrm{if}~p_{t,I_t}\geq\max_i\eta_{t,i},~
        \mathrm{then}~\tilde{c}_{t,i}=\frac{c_{t,i}}{p_{t,i}}\mathbb{I}_{i=I_t},~
        \mathrm{otherwise}~
        \tilde{c}_{t,i}=\frac{c_{t,i}\cdot\mathbb{I}_{i=I_t}}{p_{t,i}+\max_i\eta_{t,i}}.
    \end{equation}
    It is worth mentioning that
    $\tilde{c}_t$ is essentially different from that in OKS,
    and aims to ensure that
    \eqref{eq:ECML2022:OMD_Log_barrier} has a computationally efficient solution as follows
    \begin{equation}
    \label{eq:ECML2022:solution_of_OMD_log_barrier}
        \forall i\in[K],~q_{t+1,i}=
        \left(q^{-\frac{7}{8}}_{t,i}+\eta_{t,i}(\tilde{c}_{t,i}-\mu^\ast)\right)^{-\frac{8}{7}},
    \end{equation}
    where $\mu^\ast$ can be solved using binary search.
    We show more details in the supplementary materials.
    We name this algorithm IOKS (Improved OKS).

    \begin{algorithm}[!t]
        \caption{IOKS}
        \label{alg:IJCAI2022:IOKS}
        \begin{algorithmic}[1]
        \Require{$\mathcal{K}= \{\kappa_{1},\ldots,\kappa_{K}\}$, $\alpha=8$,        $\upsilon=\mathrm{e}^{\frac{2}{3\ln{T}}}$, $\eta$}
        \Ensure{$\{f_{1,i}=0,\eta_{1,i}=\eta\}^K_{i=1}$, $\mathbf{q}_1=\mathbf{p}_1=\frac{1}{K}\mathbf{1}_K$}
        \For{$t=1,\ldots,T$}
            \State Receive ${\bf x}_t$
            \State Sample a kernel $\kappa_{I_t}$ where $I_t\sim \mathbf{p}_t$
            \State Output $\hat{y}_t=f_{t,I_t}(\mathbf{x}_t)$ or $\mathrm{sign}(\hat{y}_t)$
            \State Compute $f_{t+1,I_t}$ according to \eqref{eq:IJCAI2022:IOKS:type_2_updating}
            \State Compute $\tilde{c}_{t,I_t}$ according to \eqref{eq:ECML2022:tilde_c_t:IOKS}
            \State $\forall i\in[K]$,
            compute $q_{t+1,i}$ according to \eqref{eq:ECML2022:solution_of_OMD_log_barrier}
            \State Compute $\mathbf{p}_{t+1}=(1-\delta)\mathbf{q}_{t+1}+\frac{\delta}{K}\mathbf{1}_K$
            \For{$i=1,\ldots,K$}
                \If{$\frac{1}{p_{t+1,i}}>\rho_{t,i}$}
                    \State $\rho_{t+1,i}=\frac{2}{p_{t+1,i}}$, $\eta_{t+1,i}=\upsilon\eta_{t,i}$
                \Else
                    \State $\rho_{t+1,i}=\rho_{t,i}$, $\eta_{t+1,i}=\eta_{t,i}$
                \EndIf
            \EndFor
        \EndFor
        \end{algorithmic}
    \end{algorithm}

\subsection{Regret bound}

    \begin{assumption}[Lipschitz condition]
    \label{ass:AAAI2022:the_second_lipschitz_condition}
        $\ell(\cdot,\cdot)$ is convex w.r.t. the first parameter.
        There is a constant $G_1$ such that $\forall \kappa_i\in\mathcal{K}$,
        $f\in\mathbb{H}_i$,
        $\Vert\nabla_f\ell(f(\mathbf{x}),y)\Vert_{\mathcal{H}_i} \leq G_1$.
    \end{assumption}

    \begin{theorem}
    \label{coro:IJCAI2022:IOKS:regret_bound}
        Let $\ell$ satisfy Assumption \ref{ass:AAAI2022:the_second_lipschitz_condition}.
        Let $\delta=T^{-\frac{3}{4}}$,
        \begin{equation}
        \label{eq:IJCAI2022:IOKS:learning_rates:worst_case:convex_loss}
            \eta=\frac{3\ell_{\max}K^{\frac{3}{8}}}{2UG_1\sqrt{T\ln{T}}},\quad \forall i\in[K],~
            \lambda_{t,i}=\frac{U}
            {\sqrt{2}\sqrt{1+\sum^t_{\tau=1}\Vert\tilde{\nabla}_{\tau,i}\Vert^2_{\mathcal{H}_i}}}.
        \end{equation}
        Let $T\geq 40$.
        Then the expected regret of IOKS satisfies,
        \begin{align*}
            \forall i\in[K], f\in\mathbb{H}_i,~
            \mathbb{E}\left[\bar{L}_{\mathbf{p}_{1:T}}\right]-L_T(f)
            = O\left(UG_1\sqrt{KT}\ln^{\frac{2}{3}}{T}
            +\frac{\ell^3_{\max}K^{\frac{11}{4}}}{U^2G^2_1\ln{T}}\right).
        \end{align*}
    \end{theorem}

    $\ell_{\max}$ is a normalizing constant
    and can be computed given the loss function,
    such as $\ell_{\max}\leq U+1$ in the case of absolute loss.
    Next we compare our regret bound with previous results.
    On the positive side,
    IOKS gives a $O(U\sqrt{KT}\ln^{\frac{2}{3}}{T})$ bound
    which asymptotically improves the $O(K^{\frac{1}{3}}T^{\frac{2}{3}})$ bound achieved by OKS.
    On the negative side,
    if $T$ is small, then $\sqrt{KT}\ln^{\frac{2}{3}}{T}>K^{\frac{1}{3}}T^{\frac{2}{3}}$
    and thus IOKS is slightly worse than OKS.
    $\mathrm{B(AO)_2KS}$ \cite{Liao2021High}
    which selects two hypotheses per-round,
    can provide a $\tilde{O}(\mathrm{poly}(\Vert f\Vert_{\mathcal{H}_i})\sqrt{KT})$ bound
    which is same with our result.

    We further compare the implementation of IOKS and OKS.
    It is obvious that OKS is easier than IOKS,
    since IOKS uses binary search to compute $\mathbf{q}_{t+1}$
    (see \eqref{eq:ECML2022:solution_of_OMD_log_barrier}).
    The computational cost of binary search can be omitted
    since the main computational cost comes from
    the computing of $f_{t,I_t}(\mathbf{x}_t)$ which is $O(t)$.

\section{Application to Online Kernel Selection with Time Constraint}

    In practice,
    online algorithms must face time constraint.
    We assume that there is a time budget of $\mathcal{R}$ quanta.
    Both OKS++ and IOKS suffer a $O(t)$ per-round time complexity,
    and do not satisfy the time constraint.
    In this section,
    we will use random features \cite{Rahimi2007Random} to approximate kernel functions
    and apply our two algorithms to online kernel selection with time constraint \cite{Li2022Worst}.

    We consider kernel function
    $\kappa(\mathbf{x},\mathbf{v})$ that can be decomposed as follows
    \begin{equation}
    \label{eq:IJCAI2022:definition:kernels}
        \kappa(\mathbf{x},\mathbf{v})=\int_{\Omega}\phi_\kappa(\mathbf{x},\omega)\phi_\kappa(\mathbf{v},\omega)
        \mathrm{d}\,\mu_{\kappa}(\omega),~\forall \mathbf{x},\mathbf{v}\in\mathcal{X}
    \end{equation}
    where
    $\mathcal{\phi}_\kappa:\mathcal{X}\times \Omega\rightarrow \mathbb{R}$
    is the eigenfunctions and
    $\mu_{\kappa}(\cdot)$ is a distribution function on $\Omega$.
    Let $p_{\kappa}(\cdot)$ be the density function of $\mu_{\kappa}(\cdot)$.
    We can approximate the integral via Monte-Carlo sampling.
    We sample $\{\omega_j\}^{D}_{j=1}\sim p_{\kappa}(\omega)$ independently and compute
    $
        \tilde{\kappa}(\mathbf{x},\mathbf{v})
        =\frac{1}{D}\sum^{D}_{j=1}\phi_\kappa(\mathbf{x},\omega_j)\phi_\kappa(\mathbf{v},\omega_j).
    $
    For any $f\in\mathcal{H}_{\kappa}$,
    let
    $f(\mathbf{x})=\int_{\Omega}\alpha(\omega)\phi_{\kappa}(\mathbf{x},\omega)p_{\kappa}(\omega)\mathrm{d}\,\omega$.
    We can approximate $f(\mathbf{x})$
    by $\hat{f}(\mathbf{x})=\frac{1}{D}\sum^{D}_{j=1}\alpha(\omega_j)\phi_{\kappa}(\mathbf{x},\omega_j)$.
    It can be verified that
    $\mathbb{E}[\hat{f}(\mathbf{x})]=f(\mathbf{x})$.
    Such an approximation scheme also defines
    an explicit feature mapping denoted by
    $z(\mathbf{x})=
    \frac{1}{\sqrt{D}}(\phi_{\kappa}(\mathbf{x},\omega_1),\ldots,\phi_{\kappa}(\mathbf{x},\omega_D))$.
    The approximation scheme is the so called random features \cite{Rahimi2007Random}.
    $\forall \kappa_i\in\mathcal{K}$,
    we define two hypothesis spaces \cite{Rahimi2008Ali} as follows
    \begin{align*}
        \mathbb{H}_i=&\left\{f(\mathbf{x})=\int_{\Omega}\alpha(\omega)
        \phi_{\kappa_i}(\mathbf{x},\omega)p_{\kappa_i}(\omega)\mathrm{d}\,\omega
        \left\vert \vert\alpha(\omega)\vert \leq U\right.\right\},\\
        \mathcal{F}_i=&\left\{\hat{f}(\mathbf{x})
        =\sum^{D_i}_{j=1}\alpha_j\phi_{\kappa_i}(\mathbf{x},\omega_j)
        \left\vert \vert \alpha_j\vert \leq \frac{U}{D_i}\right.\right\}.
    \end{align*}
    We can rewrite $\hat{f}(\mathbf{x})=\mathbf{w}^\top z_i(\mathbf{x})$,
    where $\mathbf{w}=\sqrt{D_i}(\alpha_1,\ldots,\alpha_{D_i})\in\mathbb{R}^{D_i}$.
    Let $\mathcal{W}_i=\{\mathbf{w}\in\mathbb{R}^{D_i}
    \vert \Vert\mathbf{w}\Vert_{\infty}\leq \frac{U}{\sqrt{D_i}}\}$.
    It can be verified that $\Vert \mathbf{w}\Vert^2_2\leq U^2$.
    For all $\kappa_i$ satisfying \eqref{eq:IJCAI2022:definition:kernels},
    there is a constant $B_i$ such that $\vert \phi_{\kappa_i}(\mathbf{x},\omega_j)\vert \leq B_i$
    for all $\omega_j\in \Omega$ and $\mathbf{x}\in\mathcal{X}$ \cite{Li2019Towards}.
    Thus we have $\vert f(\mathbf{x})\vert \leq UB_i$ for any $f\in\mathbb{H}_i$ and $f\in\mathcal{F}_i$.

    Next we define the time budget $\mathcal{R}$ and
    then present an assumption that establishes a reduction from $\mathcal{R}$ to $D_i$.

    \begin{definition}[Time Budget \cite{Li2022Worst}]
    \label{def:ALT2021:BOKS:Time_bounded_assumption}
        Let the interval of arrival time between
        $\mathbf{x}_t$ and $\mathbf{x}_{t+1}, t=1,\ldots,T$
        be less than $\mathcal{R}$ quanta.
        A time budget of $\mathcal{R}$ quanta is the maximal time interval
        that any online kernel selection algorithm
        outputs the prediction of $\mathbf{x}_t$ and $\mathbf{x}_{t+1}$.
    \end{definition}

    \begin{assumption}
    \label{assump:AAAI2021:TBOKS:time_bounded}
        For each $\kappa_i\in\mathcal{K}$ satisfying \eqref{eq:IJCAI2022:definition:kernels},
        there exist online leaning algorithms
        that can run in some $\mathcal{F}_i$ whose maximal dimension is $D_i=\beta_{\kappa_i}\mathcal{R}$
        within a time budget of $\mathcal{R}$ quanta,
        where $\beta_{\kappa_i}>0$ is a constant depending on $\kappa_i$.
    \end{assumption}
    The online gradient descent algorithm (OGD) satisfies Assumption \ref{assump:AAAI2021:TBOKS:time_bounded}.
    The main time cost of OGD comes from computing the feature mapping.
    For shift-invariant kernels,
    it requires $O(D_id)$ time complexity \cite{Rahimi2007Random}.
    For the Gaussian kernel,
    it requires $O(D_i\log(d))$ time complexity
    \cite{Le2013Fastfood,Yu2016Orthogonal}.
    Thus the per-round time complexity of OGD is linear with $D_i$.
    Since the running time of algorithm is linear with the time complexity,
    it natural to assume that $\mathcal{R}=\Theta(D_i)$.

\subsection{Algorithm}

    At any round $t$,
    our algorithm evaluates a hypothesis and avoids allocating the time budget.
    Thus we can construct $\mathcal{F}_i$ satisfying $D_i = \beta_{\kappa_i}\mathcal{R}$.
    Our algorithm is extremely simple,
    that is,
    we just need to run OKS++ or IOKS in $\{\mathcal{F}_i\}^K_{i=1}$.
    It is worth mentioning that,
    learning $\{\hat{f}_{t,i}\in\mathcal{F}_i\}^T_{t=1}$
    is equivalent to learn $\{{\bf w}^i_t\in\mathcal{W}_i\}^T_{t=1}$,
    where $\hat{f}_{t,i}(\mathbf{x}_t)=({\bf w}^i_t)^\top z_i(\mathbf{x}_t)$.
    We replace the update \eqref{eq:IJCAI2022:IOKS:type_2_updating} with
    \eqref{eq:IJCAI2022:IOKS:type_2_updating_RF},
    \begin{equation}
    \label{eq:IJCAI2022:IOKS:type_2_updating_RF}
    \begin{split}
        \tilde{{\bf w}}^i_{t+1}=&{\bf w}^i_t
        -\lambda_{t,i}\nabla_{{\bf w}^i_t}
        \ell\left(\hat{f}_{t,i}(\mathbf{x}_t),y_t\right)\frac{1}{p_{t,i}}\mathbb{I}_{i=I_t},\\
        {\bf w}^i_{t+1}=&\mathop{\arg\min}_{\mathbf{w}\in\mathcal{W}_i}
        \left\Vert \mathbf{w}-\tilde{{\bf w}}^i_{t+1}\right\Vert^2_2.
    \end{split}
    \end{equation}
    The solution of the projection operation
    in \eqref{eq:IJCAI2022:IOKS:type_2_updating_RF} is as follows,
    $$
        \forall j=1,\ldots,D_i,~w^i_{t+1,j}=\min\left\{1,
        \frac{U}{\vert \tilde{w}^i_{t+1,j}\vert \sqrt{D_i}}\right\}\tilde{w}^i_{t+1,j}.
    $$
    The time complexity of projection is $O(D_i)$
    and thus can be omitted relative to the time complexity of computing feature mapping.
    We separately name the two algorithms RF-OKS++ (Random Features for OKS++)
    and RF-IOKS (Random Features for IOKS).
    We show the pseudo-codes in the supplementary materials
    due to the space limit.
    The pseudo-codes are similar with OKS++ and IOKS.

    \begin{remark}
        The application of random features to online kernel algorithms
        is not a new idea \cite{Wang2013Large,Shen2019Random,Ghari2020Online}.
        Previous algorithms did not restrict hypothesis spaces,
        while our algorithms consider restricted hypothesis spaces,
        i.e., $\mathbb{H}_i$ and $\mathcal{F}_i$.
        This is one of the differences between our algorithms and previous algorithms.
        The restriction on the hypothesis spaces is necessary
        since we must require $\Vert {\bf w}^i_{t}\Vert_2\leq U$
        for any $i\in[K]$ and $t\in[T]$.
    \end{remark}

\subsection{Regret Bound}

    \begin{theorem}
    \label{thm:ECML2022:OKS:loss_bound:RF_OKS++}
        Let $\ell$ satisfy
        Assumption \ref{ass:AAAI2022:smoothness} with $\nu=1$
        and Assumption \ref{ass:AAAI2022:lipschitz_condition}.
        Let $\delta_t$, $\eta_t$ and $\{\lambda_{t,i}\}^K_{i=1}$
        follow Theorem \ref{thm:ECML2022:OKS:loss_bound}.
        For a fixed $\delta\in(0,1)$,
        let $\mathcal{R}$ satisfy $D_i>\frac{32}{9}C^2_0U^2B^2_i\ln\frac{1}{\delta}$,
        $\forall i\in[K]$.
        Under Assumption \ref{assump:AAAI2021:TBOKS:time_bounded},
        with probability at least $1-\delta$,
        the expected regret of RF-OKS++ satisfies
        \begin{align*}
            \forall i\in[K],~\mathbb{E}\left[\bar{L}_{\mathbf{p}_{1:T}}\right]- &L_T(f^\ast_i)
            =O\left(\frac{C_0UB_i}{\sqrt{\beta_{\kappa_i}\mathcal{R}}}L_T(f^\ast_i)\sqrt{\ln\frac{KT}{\delta}}\right.\\
            &\left.+U^{\frac{2}{3}}\left(GC_0\right)^{\frac{1}{3}}K)^{-\frac{1}{3}}
            \bar{L}^{\frac{2}{3}}_T
            +U^{\frac{2}{3}}(GC_0)^{\frac{1}{3}}K^{\frac{1}{6}}\bar{L}^{\frac{1}{6}}_T
            L^{\frac{1}{2}}_T(f^\ast_i)\right).
        \end{align*}
        Let $\ell$ satisfy
        Assumption \ref{ass:AAAI2022:smoothness} with $\nu=2$.
        Let $G=1$ in $\delta_t$ and $\lambda_{t,i}$.
        $\eta_t$ keeps unchanged.
        For a fixed $\delta\in(0,1)$,
        with probability at least $1-\delta$,
        the expected regret of RF-OKS++ satisfies
        \begin{align*}
            \forall i\in[K],~\mathbb{E}&\left[\bar{L}_{\mathbf{p}_{1:T}}\right]- L_T(f^\ast_i)
            =O\left(UB_i\frac{\sqrt{C_0TL_T(f^\ast_i)}}{\sqrt{\beta_{\kappa_i}\mathcal{R}}}\sqrt{\ln\frac{KT}{\delta}}
            \right.\\
            &\left.+\frac{C_0U^2B^2_iT}{\beta_{\kappa_i}\mathcal{R}}\ln\frac{KT}{\delta}
            +U^{\frac{2}{3}}C^{\frac{1}{3}}_0K^{-\frac{1}{3}}
            \bar{L}^{\frac{2}{3}}_T
            +U^{\frac{2}{3}}C^{\frac{1}{3}}_0K^{\frac{1}{6}}\bar{L}^{\frac{1}{6}}_TL^{\frac{1}{2}}_T(f^\ast_i)\right).
        \end{align*}
    \end{theorem}

    The regret bounds depend on $\frac{L_T(f^\ast_i)}{\sqrt{\mathcal{R}}}$ or $\frac{1}{\sqrt{\mathcal{R}}}\sqrt{TL_T(f^\ast_i)}+\frac{T}{\mathcal{R}}$.
    The larger the time budget is, the smaller the regret bound will be,
    which proves a trade-off between regret bound and time constraint.
    If $L_T(f^\ast_i)\ll T$,
    then RF-OKS++ can achieve a sublinear regret bound under a small time budget.

    \begin{theorem}
    \label{thm:IJCAI2022:RF-IOKS:regret_bound}
        Let $\ell$ satisfy Assumption \ref{ass:AAAI2022:lipschitz_condition}
        and Assumption \ref{ass:AAAI2022:the_second_lipschitz_condition}.
        Let $\{\lambda_{t,i}\}^K_{i=1}$, $\eta$ and $\delta$ follow
        Theorem \ref{coro:IJCAI2022:IOKS:regret_bound}.
        Under Assumption \ref{assump:AAAI2021:TBOKS:time_bounded},
        with probability at least $1-\delta$,
        the expected regret of RF-IOKS satisfies, $\forall i\in[K],\forall f\in\mathbb{H}_i$,
        \begin{align*}
            \mathbb{E}\left[\bar{L}_{\mathbf{p}_{1:T}}\right]-L_T(f)
            =O\left(UG_1\sqrt{KT}\ln^{\frac{2}{3}}{T}
            +\frac{\ell^3_{\max}K^{\frac{11}{4}}}{U^2G^2_1\sqrt{\ln{T}}}
            +\frac{GB_iUT}{\sqrt{\beta_{\kappa_i}\mathcal{R}}}\sqrt{\ln{\frac{KT}{\delta}}}\right).
        \end{align*}
    \end{theorem}

    The regret bound depends on $\frac{T}{\sqrt{\mathcal{R}}}$
    which also proves a trade-off between regret bound and time constraint.
    Achieving a $\tilde{O}(T^{\alpha})$ bound requires
    $\mathcal{R}=\Omega(T^{2(1-\alpha)})$, $\alpha\in[\frac{1}{2},1)$.
    The regret bounds in
    Theorem \ref{thm:ECML2022:OKS:loss_bound:RF_OKS++} depend on $L_T(f^\ast_i)$,
    while the regret bound in Theorem \ref{thm:IJCAI2022:RF-IOKS:regret_bound} depends on $T$.
    Under a same time budget $\mathcal{R}$,
    if $L_T(f^\ast_i)\ll T$,
    then RF-OKS++ enjoys better regret bounds than RF-IOKS.

\subsection{Comparison With Previous Results}

    For online kernel selection with time constraint,
    if the loss function is Lipschitz continuous,
    then there is a $\Omega(\Vert f^\ast_i\Vert_{\mathcal{H}_i}
    \max\{\sqrt{T},\frac{T}{\sqrt{\mathcal{R}}}\})$
    lower bound on expected regret \cite{Li2022Worst}.
    Theorem \ref{thm:IJCAI2022:RF-IOKS:regret_bound} gives a nearly optimal upper bound.
    LKMBooks \cite{Li2022Worst}
    gives a $O(\sqrt{T\ln{K}}+\Vert f\Vert^{2}_{\mathcal{H}_i}
    \max\{\sqrt{T},\frac{T}{\sqrt{\mathcal{R}}}\})$ bound in the case of $K\leq d$,
    and thus is slightly better than RF-IOKS.
    LKMBooks selects $K$ hypotheses per-round.
    RF-IOKS just selects a hypothesis per-round and is suitable for $K>d$.

    For smooth loss functions,
    the dominated terms in Theorem \ref{thm:ECML2022:OKS:loss_bound:RF_OKS++}
    are $O(\frac{L_T(f^\ast_i)}{\sqrt{\mathcal{R}}})$ and
    $O(\frac{1}{\sqrt{\mathcal{R}}}\sqrt{TL_T(f^\ast_i)}+\frac{T}{\mathcal{R}})$.
    If the optimal kernel $\kappa_{i^\ast}$ matches well with the data,
    that is, $L_T(f^\ast_{i^\ast})\ll T$,
    then $O(\frac{L_T(f^\ast_{i^\ast})}{\sqrt{\mathcal{R}}})$
    and $O(\frac{1}{\sqrt{\mathcal{R}}}\sqrt{TL_T(f^\ast_{i^\ast})})$
    are much smaller than $O(\frac{T}{\sqrt{\mathcal{R}}})$.
    To be specific,
    in the case of $L_T(f^\ast_{i^\ast})=o(T)$,
    RF-OKS++ is better than LKMBooks within a same time budget $\mathcal{R}$.

    Our algorithms are similar with Raker \cite{Shen2019Random}
    which also adopts random features.
    Raker selects $K$ hypotheses
    and provides a $\tilde{O}((\sqrt{\ln{K}}+\Vert f\Vert^{2}_1)\sqrt{T}
        +\Vert f\Vert_1\frac{T}{\sqrt{\mathcal{R}}})$ bound,
    where $f=\sum^T_{t=1}\alpha_t\kappa_i(\mathbf{x}_t,\cdot)$
    and $\Vert f\Vert_1=\Vert {\bm \alpha}\Vert_1$.
    The regret bounds of RF-OKS++ are better, since
    (i) they depend on $L_T(f^\ast_i)$ and $\sum^K_{j=1}L_T(f^\ast_j)$
    while the regret bound of Raker depends on $T$;
    (ii) they depend on $U$,
    while the regret bound of Raker depends on $\Vert f\Vert_1$
    which is hard to bound and explain.

\section{Experiments}

    We adopt the Gaussian kernel $\kappa(\mathbf{x},\mathbf{v})
    =\exp(-\frac{\Vert\mathbf{x}-\mathbf{v}\Vert^2_2}{2\sigma^2})$
    and select $6$ kernel widths $\sigma=2^{-2:1:3}$.
    We choose four classification datasets
    (\textit{magic04:19,020, phishing:11,055, a9a:32,561, SUSY:20,000})
    and four regression datasets
    (\textit{bank:8,192, elevators:16,599, ailerons:13,750, Hardware:28,179}).
    The datasets are downloaded from UCI
    \footnote{\url{http://archive.ics.uci.edu/ml/datasets.php}},
    LIBSVM website
    \footnote{\url{https://www.csie.ntu.edu.tw/~cjlin/libsvmtools/datasets/}}
    and WEKA.
    The features of all datasets are rescaled to fit in $[-1,1]$.
    The target variables are rescaled in $[0,1]$ for regression
    and $\{-1,1\}$ for classification.
    We randomly permutate the instances
    in the datasets 10 times and report the average results.
    All algorithms are implemented with R on a Windows machine
    with 2.8 GHz Core(TM) i7-1165G7 CPU
   \footnote{The codes are available at \url{https://github.com/JunfLi-TJU/OKS-Bandit}}.
    We separately consider online kernel selection without and with time constraint.

\subsection{Online kernel selection without time constraint}

    We compare OKS++, IOKS with OKS and aim to verify
    Theorem \ref{thm:ECML2022:OKS:loss_bound} and Theorem \ref{coro:IJCAI2022:IOKS:regret_bound}.
    We consider three loss functions:
    (i) the logistic loss satisfying
    Assumption \ref{ass:AAAI2022:smoothness} with $\nu=1$ and $C_0=1$;
    (ii) the square loss satisfying
    Assumption \ref{ass:AAAI2022:smoothness} with $\nu=2$ and $C_0=4$;
    (iii) the absolute loss which is Lipschitz continuous.
    We do not compare with $\mathrm{B(AO)_2KS}$ \cite{Liao2021High},
    since it is only used for the hinge loss.
    If $\ell$ is logistic loss,
    then we use classification datasets and measure the average mistake rate,
    i.e., $\mathrm{AMR}:=\frac{1}{T}\sum^T_{t=1}\mathbb{I}_{\hat{y}_t\neq y_t}$,
    and set $U=15$.
    Otherwise,
    we use regression datasets and measure the average loss,
    i.e., $\mathrm{AL}:=\frac{1}{T}\sum^T_{t=1}\ell(f_{t,I_t}(\mathbf{x}_t),y_t)$,
    and set $U=1$.
    The parameters of OKS++ and IOKS follow
    Theorem \ref{thm:ECML2022:OKS:loss_bound} and Theorem \ref{coro:IJCAI2022:IOKS:regret_bound}
    where we change $\eta=\frac{8\ell_{\max}K^{{3}/{8}}}{UG_1\sqrt{T\ln{T}}}$
    in Theorem \ref{coro:IJCAI2022:IOKS:regret_bound} and set $\ell_{\max}=1$.
    For OKS, we set $\delta,\lambda$ and $\eta$
    according to Remark \ref{remark:ECML2022:parameter_setting_OKS},
    where $\lambda\in\{1,5,10,25\}\cdot\sqrt{\delta/(KT)}$ and $\ell_{\max}=G=1$.
    The results are shown in Table \ref{tab:ECML2022:without_constraint:logistic_loss},
    Table \ref{tab:ECML2022:without_constraint:square_loss}
    and Table \ref{tab:ECML2022:without_constraint:absolute_loss}.

    \begin{table}[!t]
      \centering
      \setlength{\tabcolsep}{2mm}
      \caption{Online kernel selection without time constraint in the regime of logistic loss}
      \label{tab:ECML2022:without_constraint:logistic_loss}
      \begin{tabular}{l|rr|rr}
        \Xhline{0.8pt}
        \multirow{2}{*}{Algorithm}&\multicolumn{2}{c|}{phishing}&\multicolumn{2}{c}{a9a}  \\
        \cline{2-5}&AMR (\%)&Time (s)&AMR (\%)&Time (s) \\
        \hline
        OKS   & 13.80 $\pm$ 0.34 & 17.34 $\pm$ 1.48 & 19.65 $\pm$ 0.12 & 208.84 $\pm$ 31.16\\
        IOKS  & 13.25 $\pm$ 0.28 &  6.58 $\pm$ 0.18 & 17.46 $\pm$ 0.12 & 103.91 $\pm$ 13.89\\
        OKS++ & \textbf{7.80} $\pm$ \textbf{0.49}   & 32.31 $\pm$ 3.98
              & \textbf{16.57} $\pm$ \textbf{0.31}  & 474.65 $\pm$ 117.43\\
        \Xhline{0.8pt}
        \multirow{2}{*}{Algorithm}&\multicolumn{2}{c|}{magic04}&\multicolumn{2}{c}{SUSY}  \\
        \cline{2-5}&AMR (\%)&Time (s)&AMR (\%)&Time (s) \\
        \hline
        OKS   & 22.23 $\pm$ 0.22 & 6.31 $\pm$ 0.95 & 32.98 $\pm$ 0.66 & 9.97 $\pm$ 1.85\\
        IOKS  & 21.50 $\pm$ 0.18 & 4.02 $\pm$ 0.11 & 31.75 $\pm$ 0.30 & 6.68  $\pm$ 0.15\\
        OKS++ & \textbf{17.88} $\pm$ \textbf{0.57} & 11.06 $\pm$ 3.08
              & \textbf{27.84} $\pm$ \textbf{0.70} & 19.88 $\pm$ 5.28\\
        \Xhline{0.8pt}
      \end{tabular}
    \end{table}

    \begin{table}[!t]
      \centering
      \setlength{\tabcolsep}{1.7mm}
      \caption{Online kernel selection without time constraint in the regime of square loss}
      \label{tab:ECML2022:without_constraint:square_loss}
      \begin{tabular}{l|rr|rr}
        \Xhline{0.8pt}
        \multirow{2}{*}{Algorithm}&\multicolumn{2}{c|}{elevators}&\multicolumn{2}{c}{bank}  \\
        \cline{2-5}&AL &Time (s)&AL &Time (s) \\
        \hline
        OKS   & 0.0068 $\pm$ 0.0001 & 3.23 $\pm$ 0.25  & 0.0240 $\pm$ 0.0002 & 1.51 $\pm$ 0.17\\
        IOKS  & 0.0077 $\pm$ 0.0001 & 4.08 $\pm$ 0.05  & 0.0252 $\pm$ 0.0002 & 1.57 $\pm$ 0.11\\
        OKS++ & \textbf{0.0046} $\pm$ \textbf{0.0001}  & 12.75 $\pm$ 3.12
              & \textbf{0.0205} $\pm$ \textbf{0.0006}  & 4.24 $\pm$ 0.76\\
        \Xhline{0.8pt}
        \multirow{2}{*}{Algorithm}&\multicolumn{2}{c|}{ailerons}&\multicolumn{2}{c}{Hardware}  \\
        \cline{2-5}&AL &Time (s)&AL &Time (s) \\
        \hline
        OKS   & \textbf{0.0176} $\pm$ \textbf{0.0060} & 6.94 $\pm$ 0.82  & 0.0012 $\pm$ 0.0000 & 53.84 $\pm$ 1.80\\
        IOKS  & 0.0351 $\pm$ 0.0003 & 5.59 $\pm$ 0.08  & 0.0010 $\pm$ 0.0001 & 49.36 $\pm$ 1.14\\
        OKS++ & \textbf{0.0166} $\pm$ \textbf{0.0006}  & 22.79 $\pm$ 3.41
              & \textbf{0.0008} $\pm$ \textbf{0.0001}  & 114.47 $\pm$ 23.42\\
        \Xhline{0.8pt}
      \end{tabular}
    \end{table}

    \begin{table}[!t]
      \centering
      \setlength{\tabcolsep}{1.7mm}
      \caption{Online kernel selection without time constraint in the regime of absolute loss}
      \label{tab:ECML2022:without_constraint:absolute_loss}
      \begin{tabular}{l|rr|rr}
        \Xhline{0.8pt}
        \multirow{2}{*}{Algorithm}&\multicolumn{2}{c|}{elevators}&\multicolumn{2}{c}{bank}  \\
        \cline{2-5}&AL&Time (s)&AL&Time (s) \\
        \hline
        OKS   & 0.0507 $\pm$ 0.0001 & 4.76 $\pm$ 0.17 & 0.0961 $\pm$ 0.0009 & 1.55 $\pm$ 0.13\\
        IOKS  & \textbf{0.0492} $\pm$ \textbf{0.0004} & 5.20 $\pm$ 0.54 & 0.0961 $\pm$ 0.0008 & 1.64 $\pm$ 0.20\\
        \Xhline{0.8pt}
        \multirow{2}{*}{Algorithm}&\multicolumn{2}{c|}{ailerons}&\multicolumn{2}{c}{Hardware}  \\
        \cline{2-5}&AL&Time (s)&AL&Time (s) \\
        \hline
        OKS   & \textbf{0.0723} $\pm$ \textbf{0.0005} & 8.20 $\pm$ 0.19
              & \textbf{0.0105} $\pm$ \textbf{0.0001} & 56.14 $\pm$ 1.07\\
        IOKS  &0.0771 $\pm$ 0.0007 & 9.86 $\pm$ 0.68  &0.0155 $\pm$ 0.0002 & 52.01 $\pm$ 3.72\\
        \Xhline{0.8pt}
      \end{tabular}
    \end{table}

    Table \ref{tab:ECML2022:without_constraint:logistic_loss}
    and
    Table \ref{tab:ECML2022:without_constraint:square_loss} prove that
    OKS++ performs better than OKS and IOKS for smooth loss functions.
    The reason is that OKS++ adaptively tunes the parameters using the observed losses,
    while OKS and IOKS do not use this information to tune the parameters.
    The experimental results coincide with Theorem \ref{thm:ECML2022:OKS:loss_bound}.
    Besides IOKS performs similar with OKS,
    since IOKS is only asymptotically better than OKS.
    If $T$ is small, then the regret bound of OKS is smaller.
    The theoretical significance of IOKS is that
    it proves that selecting a hypothesis
    does not produce high information-theoretic cost in the worst case.

\subsection{Online kernel selection with time constraint}

    We compare RF-OKS++, RF-IOKS with Raker \cite{Shen2019Random}, LKMBooks \cite{Li2022Worst}
    and RF-OKS \cite{Yang2012Online}.
    We construct RF-OKS by combining random features with OKS.
    The parameter setting of Raker and LKMBooks follows original paper,
    except that the learning rate of Raker is chosen from
    $\{1,5,10,25\}\cdot {1}/{\sqrt{T}}$.
    The parameter setting of RF-OKS++, RF-IOKS and RF-OKS
    is same with that of OKS++, IOKS and OKS, respectively.
    We limit time budget $\mathcal{R}$ by fixing the number of random features.
    To be specific,
    we choose RF-OKS++ as the baseline
    and set $D_i=D=400$ for all $i\in[K]$ satisfying the condition in
    Theorem \ref{thm:ECML2022:OKS:loss_bound:RF_OKS++}.
    Let the average per-round running time of RF-OKS++ be $t_{\mathrm{p}}$.
    We tune $D$ or $B$ of other algorithms
    for ensuring the same running time with $t_{\mathrm{p}}$.
    The results are shown in Table \ref{tab:ECML2022:with_constraint:logistic_loss},
    Table \ref{tab:ECML2022:with_constraint:square_loss}
    and Table \ref{tab:ECML2022:with_constraint:absolute_loss}.
    In Tbale \ref{tab:ECML2022:with_constraint:absolute_loss},
    we use RF-IOKS as the baseline.

    \begin{table}[!t]
      \centering
      \setlength{\tabcolsep}{1.7mm}
      \caption{Online kernel selection with time constraint in the regime of logistic loss}
      \label{tab:ECML2022:with_constraint:logistic_loss}
      \begin{tabular}{l|c|rr|c|rr}
        \Xhline{0.8pt}
        \multirow{2}{*}{Algorithm}&\multirow{2}{*}{B\text{-}D}&\multicolumn{2}{c|}{phishing}
        &\multirow{2}{*}{B\text{-}D}&\multicolumn{2}{c}{a9a}  \\
        \cline{3-4}\cline{6-7}&&AMR (\%)&$t_p*10^5(s)$&&AMR (\%)&$t_p*10^5(s)$  \\
        \hline
        RF-OKS     & 500 & 14.61 $\pm$ 0.65 & 9.63  &450 & 21.25 $\pm$ 0.12 & 11.61\\
        LKMBooks   & 250 & 12.50 $\pm$ 1.03 & 9.46 &220 & 20.06 $\pm$ 0.54 & 11.53\\
        Raker      & 70  & 13.60 $\pm$ 1.00 & 9.35  & 90 & 24.08 $\pm$ 0.00 & 11.30\\
        \hline
        RF-IOKS    & 380 & 15.59 $\pm$ 0.39 & 9.66  &380 & 22.99 $\pm$ 0.20 & 11.95\\
        RF-OKS++   & 400 & \textbf{9.15} $\pm$ \textbf{0.56} & 9.20
                   & 400 & \textbf{17.28} $\pm$ \textbf{0.29} & 11.19 \\
        \Xhline{0.8pt}
      \end{tabular}
    \end{table}

    \begin{table}[!t]
      \centering
      \setlength{\tabcolsep}{1.7mm}
      \caption{Online kernel selection with time constraint in the regime of square loss}
      \label{tab:ECML2022:with_constraint:square_loss}
      \begin{tabular}{l|c|rr|c|rr}
        \Xhline{0.8pt}
        \multirow{2}{*}{Algorithm}&\multirow{2}{*}{B\text{-}D}&\multicolumn{2}{c|}{elevators}
        &\multirow{2}{*}{B\text{-}D}&\multicolumn{2}{c}{Hardware}  \\
        \cline{3-4}\cline{6-7}&&$\mathrm{AL}*10^2$&$t_p*10^5(s)$&&$\mathrm{AL}*10^2$&$t_p*10^5(s)$ \\
        \hline
        RF-OKS     & 450 & 0.72 $\pm$ 0.02 & 6.47 & 420 & 0.13 $\pm$ 0.00 & 10.48\\
        LKMBooks   & 220 & 0.90 $\pm$ 0.04 & 6.72 & 200 & 0.21 $\pm$ 0.01 & 10.76\\
        Raker      & 40  & 0.70 $\pm$ 0.04 & 6.57 & 80  & 0.20 $\pm$ 0.00 & 10.25\\
        \hline
        RF-IOKS    & 380 & 0.89 $\pm$ 0.01 & 6.83  & 400 & 0.12 $\pm$ 0.01 & 10.20\\
        RF-OKS++   & 400 & \textbf{0.51} $\pm$ \textbf{0.02} & 6.45
                   & 400 & \textbf{0.09} $\pm$ \textbf{0.01} & 10.31\\
        \Xhline{0.8pt}
      \end{tabular}
    \end{table}

    \begin{table}[!t]
      \centering
      \setlength{\tabcolsep}{1.5mm}
      \caption{Online kernel selection with time constraint in the regime of absolute loss}
      \label{tab:ECML2022:with_constraint:absolute_loss}
      \begin{tabular}{l|c|rr|c|rr}
        \Xhline{0.8pt}
        \multirow{2}{*}{Algorithm}&\multirow{2}{*}{B\text{-}D}&\multicolumn{2}{c|}{elevators}
        &\multirow{2}{*}{B\text{-}D}&\multicolumn{2}{c}{Hardware}  \\
        \cline{3-4}\cline{6-7}&&AL&$t_p*10^5$&&AL&$t_p*10^5$ \\
        \hline
        RF-OKS     &530 & \textbf{0.0515} $\pm$ \textbf{0.0004} &  7.13
                   &400 & \textbf{0.0108} $\pm$ \textbf{0.0001} & 10.39\\
        LKMBooks   &230 & 0.0550 $\pm$ 0.0014 &  7.35 &200 & 0.0203 $\pm$ 0.0020 & 10.41\\
        Raker      &50  & 0.0550 $\pm$ 0.0012 &  7.41 &80  & 0.0154 $\pm$ 0.0001 & 10.37\\
        \hline
        RF-IOKS    &400 & \textbf{0.0515} $\pm$ \textbf{0.0007} &  7.63 &400 & 0.0164 $\pm$ 0.0002 & 10.97\\
        \Xhline{0.8pt}
      \end{tabular}
    \end{table}

    For smooth loss functions,
    RF-OKS++ still performs best under a same time budget.
    The reason is also that
    RF-OKS++ adaptively tunes the parameters using the observed losses,
    while the other algorithms do not use the observed losses.
    For the square loss function,
    Theorem \ref{thm:ECML2022:OKS:loss_bound:RF_OKS++}
    shows the regret bound depends on $O(\frac{1}{\sqrt{\mathcal{R}}}\sqrt{TL_T(f^\ast_i)})$
    which becomes $O(\frac{T}{\sqrt{\mathcal{R}}})$ in the worst case
    and thus is same with previous results.
    To explain the contradiction,
    we recorded the cumulative square losses of RF-OKS++,
    i.e., $\sum^T_{t=1}(f_{t,I_t}(\mathbf{x}_t)-y_t)^2$
    and use it as a proxy for $L_T(f^\ast_i)$.
    In our experiments,
    $L_T(f^\ast_i)\approx88.6$ on the \textit{elevators} dataset and
    $L_T(f^\ast_i)\approx23.8$ on the \textit{Hardware} dataset.
    Thus $L_T(f^\ast_i)\ll T$
    and $O(\frac{1}{\sqrt{\mathcal{R}}}\sqrt{TL_T(f^\ast_i)})$
    is actually smaller than $O(\frac{T}{\sqrt{\mathcal{R}}})$.
    The above results coincide with Theorem \ref{thm:ECML2022:OKS:loss_bound:RF_OKS++}.

    RF-IOKS shows similar performance with the baseline algorithms,
    which is consistent with Theorem \ref{thm:IJCAI2022:RF-IOKS:regret_bound}.
    The regret bound of RF-IOKS is slightly worse than that of LKMBooks and Raker,
    and is only asymptotically better than RF-OKS.
    All of the baseline algorithms tune the stepsize in hindsight,
    which is impossible in practice since the data can only be predicted once.
    RF-IOKS also proves that selecting a hypothesis
    does not damage the regret bound much in the worst case.
    More experiments are shown in the supplementary materials.

\section{Conclusion}

    In this paper,
    we have proposed two algorithms for online kernel selection under bandit feedback
    and improved the previous worst-case regret bound.
    OKS++ which is applied for smooth loss functions,
    adaptively tunes parameters of OKS and achieves data-dependent regret bounds
    depending on the minimal cumulative losses.
    IOKS which is applied for Lipschitz loss functions,
    achieves a worst-case regret bound asymptotically better than previous result.
    We further apply the two algorithms to online kernel selection with time constraint
    and obtain better or similar regret bounds.

    From the perspective of algorithm design,
    there is a trade-off between regret bound and the amount of observed information.
    IOKS proves that
    selecting a hypothesis or multiple hypotheses per-round
    will not induce significant variation on the worst-case regret bound.
    OKS++ which performs well both in theory and practice,
    implies that there may be differences in terms of data-dependent regret bounds.
    This question is left to future work.


%
%
%
%

\appendix

\section{Proof of Theorem \ref{thm:ECML2022:OKS:loss_bound}}

    We will bound the regret by the cumulative losses of any hypothesis $f\in\mathbb{H}_i$.
    \begin{align*}
        &\sum^T_{t=1}\ell(f_{t,I_t}(\mathbf{x}_t),y_t)-\frac{\ell(f_{t,i}(\mathbf{x}_t),y_t)}{p_{t,i}}
        \mathbb{I}_{I_t=i}
        +\sum^T_{t=1}\frac{\ell(f_{t,i}(\mathbf{x}_t),y_t)-\ell(f(\mathbf{x}_t),y_t)}{p_{t,i}}\mathbb{I}_{I_t=i}\\
        =&
        \sum^T_{t=1}[c_{t,I_t}-\tilde{c}_{t,i}]+
        \sum^T_{t=1}\frac{\ell(f_{t,i}(\mathbf{x}_t),y_t)-\ell(f(\mathbf{x}_t),y_t)}{p_{t,i}}\mathbb{I}_{I_t=i}\\
        =&
        \underbrace{\sum^T_{t=1}\langle\mathbf{p}_t-\mathbf{u},\tilde{c}_t\rangle}_{\Xi_1}+
        \underbrace{\sum^T_{t=1}\frac{\ell(f_{t,i}(\mathbf{x}_t),y_t)
        -\ell(f(\mathbf{x}_t),y_t)}{p_{t,i}}\mathbb{I}_{I_t=i}}_{\Xi_2},
    \end{align*}
    where $\mathbf{u}\in\Delta_{K}$.
    We can set $\mathbf{u}=\mathbf{e}_i$.

\subsection{Analyzing $\Xi_1$}

    To start with, we prove a technique lemma.
    \begin{lemma}
    \label{lemma:ECML2022_supp:auxiliary_inequality_1}
        Let $a\geq b>0$. Then
        $$
            \frac{a-b}{a^{\frac{1}{3}}} \leq 2(a^{\frac{2}{3}}-b^{\frac{2}{3}}),~\mathrm{and}~
            \frac{a-b}{a^{\frac{1}{2}}} \leq 2(a^{\frac{1}{2}}-b^{\frac{1}{2}}).
        $$
    \end{lemma}
    \begin{proof}[of Lemma \ref{lemma:ECML2022_supp:auxiliary_inequality_1}]
        Multiplying by $a^{\frac{1}{3}}$ (and $a^{\frac{1}{2}}$)
        on both sides and rearranging terms yields
        $a+b\geq 2a^{\frac{1}{3}}b^{\frac{2}{3}}$
        (and $a+b\geq 2a^{\frac{1}{2}}b^{\frac{1}{2}}$).
        It is easy to verify that
        $$
            a+b\geq 2a^{\frac{1}{2}}b^{\frac{1}{2}}
            =2(a^{\frac{1}{3}}b^{\frac{2}{3}})\cdot (a^{\frac{1}{6}}b^{-\frac{1}{6}})
            \geq 2a^{\frac{1}{3}}b^{\frac{2}{3}},
        $$
        where the last inequality comes from $a\geq b$.
        Thus we conclude the proof.
    \end{proof}
    Recalling that $\mathbf{p}_t=(1-\delta_t)\mathbf{q}_t+\delta_t\frac{\mathbf{1}_K}{K}$.
    We decompose $\Xi_1$ as follows,
    \begin{align*}
        \Xi_1=\sum^T_{t=1}\langle \mathbf{q}_t-\mathbf{u},\tilde{c}_t\rangle
        +&\sum^T_{t=1}\langle \mathbf{p}_t-\mathbf{q}_t,\tilde{c}_t\rangle
        =\sum^T_{t=1}\langle \mathbf{q}_t-\mathbf{u},\tilde{c}_t\rangle
        +\sum^T_{t=1}\delta_t\langle \frac{\mathbf{1}_K}{K}-\mathbf{q}_t,\tilde{c}_t\rangle\\
        \leq&\sum^T_{t=1}\langle \mathbf{q}_t-\mathbf{u},\tilde{c}_t\rangle
        +\frac{1}{K}\sum^T_{t=1}\frac{(GC_0)^{\frac{1}{3}}(UK)^{\frac{2}{3}}\sum^K_{i=1}\tilde{c}_{t,i}}
    {2\max\left\{(GC_0)^{\frac{1}{3}}(UK)^{\frac{2}{3}},
        2\tilde{C}^{\frac{1}{3}}_{t,K}\right\}}.
    \end{align*}
    Let $t_0$ be the last time instance such that
    $(GC_0)^{\frac{1}{3}}(UK)^{\frac{2}{3}}\geq
        2\tilde{C}^{\frac{1}{3}}_{t_0,K}$.
    We consider two cases.
    For $t=1,\ldots,t_0$, we have
    $$
        \frac{1}{K}\sum^{t_0}_{t=1}\frac{(GC_0)^{\frac{1}{3}}(UK)^{\frac{2}{3}}\sum^K_{i=1}\tilde{c}_{t,i}}
    {2\max\left\{(GC_0)^{\frac{1}{3}}(UK)^{\frac{2}{3}},
        2\tilde{C}^{\frac{1}{3}}_{t_0,K}\right\}}
        =\frac{1}{2K}\tilde{C}_{t_0,K}.
    $$
    For $t=t_0+1,\ldots,T$,
    let $a=\sum^t_{\tau=1}\sum^K_{i=1}\tilde{c}_{\tau,i}^{\frac{1}{3}}$
    and $b=\sum^{t-1}_{\tau=1}\sum^K_{i=1}\tilde{c}_{\tau,i}^{\frac{1}{3}}$
    in Lemma \ref{lemma:ECML2022_supp:auxiliary_inequality_1}.
    Then we obtain
    $$
        \frac{(GC_0)^{\frac{1}{3}}(UK)^{\frac{2}{3}}}{4K}\cdot
        \sum^T_{t=t_0+1}\frac{\sum^K_{i=1}\tilde{c}_{t,i}}
    {\sum^t_{\tau=1}\sum^K_{i=1}\tilde{c}_{\tau,i}^{\frac{1}{3}}}
    \leq \frac{(GC_0)^{\frac{1}{3}}(UK)^{\frac{2}{3}}}{2K}
    \left(\tilde{C}^{\frac{2}{3}}_{T,K}-\tilde{C}^{\frac{2}{3}}_{t_0,K}\right).
    $$
    Combining the two cases, we obtain
    \begin{align*}
    &\frac{1}{K}\sum^T_{t=1}\frac{(GC_0)^{\frac{1}{3}}(UK)^{\frac{2}{3}}\sum^K_{i=1}\tilde{c}_{t,i}}
    {2\max\left\{(GC_0)^{\frac{1}{3}}(UK)^{\frac{2}{3}},
        2\tilde{C}^{\frac{1}{3}}_{t,K}\right\}}
        \leq \frac{(GC_0)^{\frac{1}{3}}(UK)^{\frac{2}{3}}}{2K}\tilde{C}^{\frac{2}{3}}_{T,K},\\
        &\Xi_1
        \leq\sum^T_{t=1}\langle \mathbf{q}_t-\mathbf{u},\tilde{c}_t\rangle
        +\frac{1}{2}U^{\frac{2}{3}}\left(\frac{GC_0}{K}\right)^{\frac{1}{3}}
        \tilde{C}^{\frac{2}{3}}_{T,K}.
    \end{align*}
    To analyze $\Xi_1$,
    we redefine the updating of $\mathbf{q}_t$ in online mirror descent (OMD) framework,
    that is,
    \begin{align*}
        \nabla_{\mathbf{q}'_{t+1}}\psi_t(\mathbf{q}'_{t+1})
        =&\nabla_{\mathbf{q}_{t}}\psi_t(\mathbf{q}_{t})-\tilde{c}_t,\\
        \mathbf{q}_{t+1}=&\mathop{\arg\min}_{\mathbf{q}\in\Delta_{K}}
        \mathcal{D}_{\psi_t}(\mathbf{q},\mathbf{q}'_{t+1}),
    \end{align*}
    where $\psi_t(\mathbf{p})=\sum^K_{i=1}\frac{p_i}{\eta_t}\ln{p_i}$.
    The Bregman divergence defined on negative entropy regularizer is
    $\mathcal{D}_{\psi_{t}}(\mathbf{u},\mathbf{v})
        =\sum^K_{i=1}\frac{1}{\eta_t}\left(u_i\ln{\frac{u_i}{v_i}}-u_i+v_i\right).$\\
    The first step of OMD is equivalent to
    $$
        \forall i\in[K],~q'_{t+1,i} = q_{t,i}\exp\left(-\eta_t\tilde{c}_{t,i}\right).
    $$
    The second step of OMD is the solution of the following problem
    $$
        \min_{\mathbf{q}\in\Delta_{K}} \sum^K_{i=1}\frac{1}{\eta}
        \left(q_i\ln{\frac{q_i}{q'_{t+1,i}}}-q_i+q'_{t+1,i}\right)
        +\lambda\left(\sum^K_{i=1}q_i-1\right)-\sum^K_{i=1}\beta_iq_{t+1,i}.
    $$
    The optimal solution is $\mathbf{q}_{t+1}$, $\lambda^\ast$ and $\{\beta^\ast_i\}^K_{i=1}$.
    Solving the optimal problem gives
    $$
        q_{t+1,i}=\frac{q'_{t+1,i}}{\sum^K_{j=1}q'_{t+1,j}}
        =\frac{q_{t,i}\exp\left(-\eta_t\tilde{c}_{t,i}\right)}
        {\sum^K_{j=1}q_{t,j}\exp\left(-\eta_t\tilde{c}_{t,j}\right)},
    $$
    which coincides the updating of OKS.

    We decompose $\langle \mathbf{q}_t-\mathbf{u},\tilde{c}_t\rangle$ as follows,
    \begin{align*}
        \langle \mathbf{q}_t-\mathbf{u},\tilde{c}_t\rangle
        =&\langle \mathbf{q}'_{t+1}-\mathbf{u},\tilde{c}_t\rangle
        +\langle \mathbf{q}_t-\mathbf{q}'_{t+1},\tilde{c}_t\rangle\\
        =&\mathcal{D}_{\psi_t}(\mathbf{u},\mathbf{q}_t)-\mathcal{D}_{\psi_t}(\mathbf{u},\mathbf{q}'_{t+1})
        -\mathcal{D}_{\psi_t}(\mathbf{q}'_{t+1},\mathbf{q}_t)
        +\langle \mathbf{q}_t-\mathbf{q}'_{t+1},\tilde{c}_t\rangle\\
        \leq&\underbrace{\mathcal{D}_{\psi_t}(\mathbf{u},\mathbf{q}_t)
        -\mathcal{D}_{\psi_t}(\mathbf{u},\mathbf{q}_{t+1})}_{\Xi_{1,1}}
        \underbrace{-\mathcal{D}_{\psi_t}(\mathbf{q}'_{t+1},\mathbf{q}_t)
        +\langle \mathbf{q}_t-\mathbf{q}'_{t+1},\tilde{c}_t\rangle}_{\Xi_{1,2}}.
    \end{align*}
    We analyze $\Xi_{1,1}$ and $\Xi_{1,2}$ respectively.\\
    Summing over $t=1,\ldots,T$ gives
    \begin{align*}
        \sum^T_{t=1}\Xi_{1,1}
        =&\mathcal{D}_{\psi_1}(\mathbf{u},\mathbf{q}_1)
        +\sum^T_{t=2}[\mathcal{D}_{\psi_t}(\mathbf{u},\mathbf{q}_t)
        -\mathcal{D}_{\psi_{t-1}}(\mathbf{u},\mathbf{q}_{t})]
        -\mathcal{D}_{\psi_{T}}(\mathbf{u},\mathbf{q}_{T+1})\\
        \leq&\frac{\ln{K}}{\eta_1}+\sum^T_{t=2}\sum^K_{i=1}\left(u_i\ln{\frac{u_i}{q_{t,i}}}-u_i+q_{t,i}\right)
        \left(\frac{1}{\eta_t}-\frac{1}{\eta_{t-1}}\right)\\
        \leq&\frac{\ln{K}}{\eta_1}+
        \sum^T_{t=2}\left(\frac{1}{\eta_t}-\frac{1}{\eta_{t-1}}\right)\max_t\ln{\frac{K}{\delta_t}}\\
        \leq&\frac{1}{3\eta_T}\ln{\frac{64K\tilde{C}_{T,K}}{U^2GC_0}}
        \leq\frac{\sqrt{2\ell_{\max}\tilde{C}_{T,K}}}{3\sqrt{2\ln{K}}}
        \ln{\frac{64K\tilde{C}_{T,K}}{U^2GC_0}}\quad (\frac{q_{t,i}}{p_{t,i}}\leq 2),
    \end{align*}
    where $\ell_{\max}=\max_{t,i}c_{t,i}$.
    The second inequality comes from the fact $q_{t,i}\geq \frac{\delta_t}{K}$.
    For $\Xi_{1,2}$, we have
    \begin{align*}
        \sum^T_{t=1}\Xi_{1,2}
        =&\sum^T_{t=1}\left(\langle\mathbf{q}_t-\mathbf{q}'_{t+1},\tilde{c}_t\rangle
        -\sum^K_{i=1}\frac{1}{\eta_t}\left(q'_{t+1,i}\ln\frac{q'_{t+1,i}}{q_{t,i}}
        -q'_{t+1,i}+q_{t,i}\right)\right)\\
        =&\sum^T_{t=1}\left(\langle\mathbf{q}_t,\tilde{c}_t\rangle+
        \sum^K_{i=1}\frac{1}{\eta_t}\left(q_{t,i}\exp\left(-\eta_t\tilde{c}_{t,i}\right)-q_{t,i}\right)\right)\\
        \leq&\sum^T_{t=1}\left(\langle\mathbf{q}_t,\tilde{c}_t\rangle+
        \sum^K_{i=1}\frac{1}{\eta_t}\left(q_{t,i}\left(-\eta_t\tilde{c}_{t,i}
        +\frac{1}{2}\eta^2_t\tilde{c}^2_{t,i}\right)\right)\right)\\
        \leq&\frac{1}{2}\sum^T_{t=1}\frac{\sqrt{2\ln{K}}\sum^K_{i=1} q_{t,i}\tilde{c}^2_{t,i}}
        {\sqrt{1+\sum^t_{\tau=1}\sum^K_{i=1}q_{\tau,i}\tilde{c}^2_{\tau,i}}}
        \leq 2\sqrt{\ell_{\max}\tilde{C}_{T,K}\ln{K}},
    \end{align*}
    where the first inequality comes from $\exp(-x)\leq 1-x+\frac{1}{2}x^2$ for $x\geq 0$.
    The last inequality comes from Lemma \ref{lemma:ECML2022_supp:auxiliary_inequality_1}.
    Combining all results gives
    \begin{align*}
        \Xi_1 \leq& 2\sqrt{\ell_{\max}\tilde{C}_{1:K}\ln{K}}
        +\frac{\sqrt{\ell_{\max}\tilde{C}_{T,K}}}{3\sqrt{\ln{K}}}\ln{\frac{64K\tilde{C}_{T,K}}{U^2GC_0}}
        +\frac{1}{2}U^{\frac{2}{3}}\left(\frac{GC_0}{K}\right)^{\frac{1}{3}}
        \tilde{C}^{\frac{2}{3}}_{T,K}\\
        \lesssim& U^{\frac{2}{3}}\left(\frac{GC_0}{K}\right)^{\frac{1}{3}}\tilde{C}^{\frac{2}{3}}_{T,K},
    \end{align*}
    where we assume $\tilde{C}_{T,K}>K^2$.
    We will analyze the case $\tilde{C}_{T,K}\leq K^2$ later.

\subsection{Analyzing $\Xi_2$}
\label{sec:ECML2022_supp:theorem_1:xi_2}

    Using the convexity of loss function, we have
    \begin{align*}
        \Xi_2
        \leq&\sum^T_{t=1}\left\langle \nabla_{t,i},f_{t,i}-f\right\rangle\cdot\frac{\mathbb{I}_{I_t=i}}{p_{t,i}}
        =\sum^T_{t=1}\left\langle \tilde{\nabla}_{t,i},f_{t,i}-f\right\rangle\\
        =&\sum^T_{t=1}\frac{\Vert f_{t,i}\Vert^2_{\mathcal{H}_i}
        -\Vert f_{t+1,i}\Vert^2_{\mathcal{H}_i}
        +2\langle f_{t,i},-f\rangle-2\langle f_{t+1,i},-f\rangle}{2\lambda_{t,i}}
        +\frac{\lambda_{t,i}}{2}\left\Vert\tilde{\nabla}_{t,i}\right\Vert^2_{\mathcal{H}_i}\\
        \leq&\frac{3U^2}{2\lambda_{T,i}}
        +\frac{KGC_0}{2}
        \sum^T_{t=1}\frac{\lambda_{t,i}}{\delta_t}
        \frac{\ell(f_{t,i}(\mathbf{x}_{t}),y_{t})}{p_{t,i}}\mathbb{I}_{I_t=i}.
    \end{align*}
    The last inequality comes from $p_{t,i}\geq \frac{\delta_t}{K}$ and 
    Assumption \ref{ass:AAAI2022:lipschitz_condition}.
    Let
    \begin{equation}
    \label{eq:ECML2022:proof_theorem1_value_lambda}
        \lambda_{t,i}= \frac{U^{\frac{4}{3}}(\max\left\{GC_0U^2K^2,8\tilde{C}_{t,K}\right\})^{-\frac{1}{6}}}
        {\sqrt{4/3}K^{\frac{1}{6}}(GC_0)^{\frac{1}{3}}\sqrt{1+\Delta_{t,i}}},~
        \Delta_{t,i}=\sum^t_{\tau=1}
        \frac{\ell(f_{\tau,i}(\mathbf{x}_{\tau}),y_{\tau})}{p_{\tau,i}}\mathbb{I}_{I_{\tau}=i}.
    \end{equation}
    For $t=1,\ldots,t_0$, we have
    $$
        KGC_0\sum^{t_0}_{t=1}\frac{\lambda_{t,i}}{\delta_t}\tilde{\ell}(f_{t,i}(\mathbf{x}_{t}),y_{t})
        \leq \sqrt{3}U^{\frac{2}{3}}K^{\frac{1}{6}}(GC_0)^{\frac{1}{3}}\cdot
        U^{\frac{1}{3}}K^{\frac{1}{3}}(GC_0)^{\frac{1}{6}}
        \sum^{t_0}_{t=1}\frac{\tilde{\ell}(f_{t,i}(\mathbf{x}_{t}),y_{t})}{\sqrt{1+\Delta_{t,i}}},
    $$
    where $\tilde{\ell}(f_{t,i}(\mathbf{x}_{t}),y_{t}):=
    \frac{\ell(f_{t,i}(\mathbf{x}_{t}),y_{t})}{p_{t,i}}\mathbb{I}_{I_t=i}$.
    For $t=t_0+1,\ldots,T$, we have
    $$
        KGC_0\sum^{t_0}_{t=1}\frac{\lambda_{t,i}}{\delta_t}\tilde{\ell}(f_{t,i}(\mathbf{x}_{t}),y_{t})
        \leq \sqrt{3}U^{\frac{2}{3}}K^{\frac{1}{6}}(GC_0)^{\frac{1}{3}}\cdot\sqrt{2}\tilde{C}^{\frac{1}{6}}_{T,K}
        \sum^{t_0}_{t=1}\frac{\tilde{\ell}(f_{t,i}(\mathbf{x}_{t}),y_{t})}{\sqrt{1+\Delta_{t,i}}}.
    $$
    Next we assume that
    \begin{equation}
    \label{eq:ECML2022:condition_1}
        \tilde{C}_{T,K}\geq \frac{1}{8}U^2K^2GC_0.
    \end{equation}
    Using Assumption \ref{ass:AAAI2022:smoothness} and 
    Assumption \ref{ass:AAAI2022:lipschitz_condition},
    we further obtain
    \begin{align}
        \Xi_2
        \leq&\frac{3U^2}{2\lambda_{T,i}}
        +\sqrt{3}U^{\frac{2}{3}}K^{\frac{1}{6}}(GC_0)^{\frac{1}{3}}\tilde{C}^{\frac{1}{6}}_{T,K}
        \sum^T_{t=1}\left[\sqrt{1+\Delta_{t,i}}-\sqrt{1+\Delta_{t-1,i}}\right]\nonumber\\
        \leq&2\sqrt{3}U^{\frac{2}{3}}K^{\frac{1}{6}}(GC_0)^{\frac{1}{3}}\tilde{C}^{\frac{1}{6}}_{T,K}
        \sqrt{\sum^T_{t=1}\frac{\ell(f_{t,i}(\mathbf{x}_t),y_t)}{p_{t,i}}\mathbb{I}_{I_t=i}}
        \label{eq:ECML2022:proof_of_Theorem_1:Xi_2}.
    \end{align}
    Let $\sum^T_{t=1}\frac{\ell(f_{t,i}(\mathbf{x}_t),y_t)}{p_{t,i}}\mathbb{I}_{I_t=i}$
    be the unknown parameter.
    Solving \eqref{eq:ECML2022:proof_of_Theorem_1:Xi_2} gives
    $$
        \Xi_2
        \leq 12U^{\frac{4}{3}}K^{\frac{1}{3}}(GC_0)^{\frac{2}{3}}\tilde{C}_{T,K}^{\frac{1}{3}}
        +2\sqrt{3}U^{\frac{2}{3}}K^{\frac{1}{6}}\tilde{C}_{T,K}^{\frac{1}{6}}(GC_0)^{\frac{1}{3}}
        \sqrt{\sum^T_{t=1}\frac{\ell(f(\mathbf{x}_t),y_t)}{p_{t,i}}\mathbb{I}_{I_t=i}}.
    $$

\subsection{Combining $\Xi_1$ with $\Xi_2$}

    \begin{equation}
    \label{eq:ECML2022:proof_theorem1:final_regret_eq1}
    \begin{split}
        &\sum^T_{t=1}\ell(f_{t,I_t}(\mathbf{x}_t),y_t)
        -\sum^T_{t=1}\frac{\ell(f(\mathbf{x}_t),y_t)}{p_{t,i}}\mathbb{I}_{I_t=i}
        \leq U^{\frac{2}{3}}\left(GC_0\right)^{\frac{1}{3}}K^{-\frac{1}{3}}\tilde{C}^{\frac{2}{3}}_{T,K}\\
        +&12U^{\frac{4}{3}}K^{\frac{1}{3}}(GC_0)^{\frac{2}{3}}\tilde{C}_{T,K}^{\frac{1}{3}}
        +2\sqrt{3}U^{\frac{2}{3}}K^{\frac{1}{6}}\tilde{C}_{T,K}^{\frac{1}{6}}(GC_0)^{\frac{1}{3}}
        \sqrt{\sum^T_{t=1}\frac{\ell(f(\mathbf{x}_t),y_t)}{p_{t,i}}\mathbb{I}_{I_t=i}}.
    \end{split}
    \end{equation}
    Next we need to bound $\tilde{C}_{T,K}=\sum^T_{t=1}\tilde{c}_{t,I_t}=
    \sum^T_{t=1}\frac{\ell(f_{t,I_t}(\mathbf{x}_t),y_t)}{p_{t,I_t}}$.
    \begin{lemma}
    \label{lemma:ECML2022_supp:auxiliary_inequality_2}
        Let $x\geq 0$. Let $y\geq 0$ and $a,b>0$ be constant. If $x-by\leq ax^{\frac{2}{3}}$,
        then $x -by< (4b)^{\frac{2}{3}}ay^{\frac{2}{3}}+(\frac{4}{3})^3a^3$.
    \end{lemma}
    \begin{proof}[of Lemma \ref{lemma:ECML2022_supp:auxiliary_inequality_2}]
        We consider two cases.
        For all $x$ satisfying $ax^{\frac{2}{3}}>3by$, we further have $x\leq \frac{4}{3}ax^{\frac{2}{3}}$.
        Rearranging terms yields $x\leq (\frac{4}{3})^3a^3$.
        For all $x$ satisfying $ax^{\frac{2}{3}}\leq 3by$, we further have $x\leq 4by$.
        Substituting into the condition, we have $x-by\leq a(4by)^{\frac{2}{3}}$.
        Combining with the two cases concludes the proof.
    \end{proof}
    \begin{lemma}
    \label{lemma:ECML2022:proof_theorem1}
        Let $\eta_t$, $\delta_t$ and $\{\lambda_{t,i}\}^K_{i=1}$ follow Theorem 2.
        Then OKS++ guarantees
        $$
        \sum^T_{t=1}\frac{\ell(f_{t,I_t}(\mathbf{x}_t),y_t)}{p_{t,I_t}}
        \leq \tilde{L}_{1:K}(f)+8K^{\frac{2}{3}}U^{\frac{2}{3}}(GC_0)^{\frac{1}{3}}
        \tilde{L}^{\frac{2}{3}}_{1:K}(f)+99U^2K^2GC_0,
        $$
        where $\tilde{L}_{1:K}(f):=\sum^K_{i=1}\sum^T_{t=1}\frac{\ell(f(\mathbf{x}_t),y_t)}{p_{t,i}}\mathbb{I}_{I_t=i}$.
    \end{lemma}
    \begin{proof}[of Lemma \ref{lemma:ECML2022:proof_theorem1}]
        Assuming that condition \eqref{eq:ECML2022:condition_1} holds.
        Then we start with \eqref{eq:ECML2022:proof_of_Theorem_1:Xi_2}.
        Summing over $i=1,\ldots,K$ yields
        \begin{align*}
            &\sum^K_{i=1}\sum^T_{t=1}\frac{\ell(f_{t,i}(\mathbf{x}_t),y_t)
            -\ell(f(\mathbf{x}_t),y_t)}{p_{t,i}}\mathbb{I}_{I_t=i}\\
            \leq& \sum^K_{i=1}2\sqrt{3}
            U^{\frac{2}{2}}K^{\frac{1}{6}}\tilde{C}^{\frac{1}{6}}_{T,K}(GC_0)^{\frac{1}{3}}
            \sqrt{\sum^T_{t=1}\frac{\ell(f_{t,i}(\mathbf{x}_t),y_t)}{p_{t,i}}\mathbb{I}_{I_t=i}}\\
            \leq& 2\sqrt{3}U^{\frac{2}{3}}K^{\frac{1}{6}}(GC_0)^{\frac{1}{3}}
            \sqrt{K(\sum^T_{t=1}\sum^K_{i=1}\frac{\ell(f_{t,i}(\mathbf{x}_t),y_t)}
            {p_{t,i}}\mathbb{I}_{I_t=i})^{\frac{1}{3}}
            \sum^K_{i=1}\sum^T_{t=1}\frac{\ell(f_{t,i}(\mathbf{x}_t),y_t)}{p_{t,i}}\mathbb{I}_{I_t=i}}\\
            =&2\sqrt{3}U^{\frac{2}{3}}K^{\frac{2}{3}}(GC_0)^{\frac{1}{3}}
            \left(\sum^K_{i=1}\sum^T_{t=1}\frac{\ell(f_{t,i}(\mathbf{x}_t),y_t)}{p_{t,i}}\mathbb{I}_{I_t=i}\right)^{\frac{2}{3}}.
        \end{align*}
        According to Lemma \ref{lemma:ECML2022_supp:auxiliary_inequality_2},
        solving the inequality gives
        $$
            \sum^K_{i=1}\sum^T_{t=1}\frac{\ell(f_{t,i}(\mathbf{x}_t),y_t)}{p_{t,i}}\mathbb{I}_{I_t=i}
            \leq \tilde{L}_{1:K}(f)+8K^{\frac{2}{3}}U^{\frac{2}{3}}(GC_0)^{\frac{1}{3}}\tilde{L}^{\frac{2}{3}}_{1:K}(f)
            +99U^2K^2GC_0.
        $$
        Recalling that
        $\frac{\ell(f_{t,I_t}(\mathbf{x}_t),y_t)}{p_{t,I_t}}
        =\sum^K_{i=1}\frac{\ell(f_{t,i}(\mathbf{x}_t),y_t)}{p_{t,i}}\mathbb{I}_{I_t=i}$.
        If condition \eqref{eq:ECML2022:condition_1} does not hold,
        that is, $\tilde{C}_{T,K} < \frac{1}{8}U^2K^2GC_0$.
        Combining the two cases concludes the proof.
    \end{proof}

\subsection{The final regret}

    Substituting the inequality in Lemma \ref{lemma:ECML2022:proof_theorem1}
    into \eqref{eq:ECML2022:proof_theorem1:final_regret_eq1}
    and omitting the lower order terms,
    we obtain
    \begin{align*}
        &\sum^T_{t=1}\ell(f_{t,I_t}(\mathbf{x}_t),y_t)
        -\sum^T_{t=1}\frac{\ell(f(\mathbf{x}_t),y_t)}{p_{t,i}}\mathbb{I}_{I_t=i}\\
        =&O\left(U^{\frac{2}{3}}\left(GC_0\right)^{\frac{1}{3}}K^{-\frac{1}{3}}\tilde{C}^{\frac{2}{3}}_{1:K}
        +U^{\frac{2}{3}}K^{\frac{1}{6}}\tilde{C}_{1:K}^{\frac{1}{6}}(GC_0)^{\frac{1}{3}}
        \sqrt{\sum^T_{t=1}\frac{\ell(f(\mathbf{x}_t),y_t)}{p_{t,i}}\mathbb{I}_{I_t=i}}\right)\\
        =&O\left(U^{\frac{2}{3}}\left(\frac{GC_0}{K}\right)^{\frac{1}{3}}\tilde{L}^{\frac{2}{3}}_{1:K}(f)
        +U^{\frac{2}{3}}(GC_0)^{\frac{1}{3}}K^{\frac{1}{6}}\tilde{L}_{1:K}^{\frac{1}{6}}(f)
        \sqrt{\sum^T_{t=1}\frac{\ell(f(\mathbf{x}_t),y_t)}{p_{t,i}}\mathbb{I}_{I_t=i}}\right).
    \end{align*}
    Finally, we need to take expectation w.r.t. $\{I_t\}^T_{t=1}$.
    Note that $\mathbb{E}_t(\frac{\ell(f(\mathbf{x}_t),y_t)}{p_{t,i}}\mathbb{I}_{I_t=i})
    =\ell(f(\mathbf{x}_t),y_t)$ for any fixed $f$.
    The conditional expectation can be canceled since $\ell(f(\mathbf{x}_t),y_t)$
    is independent of $\{I_{\tau}\}^{t-1}_{\tau=1}$.
    Besides,
    $$
        \mathbb{E}_t[\ell(f_{t,I_t}(\mathbf{x}_t),y_t)\vert I_1,\ldots,I_{t-1}]
    =\mathbb{E}_t\left[\sum^K_{i=1}p_{t,i}\ell(f_{t,i}(\mathbf{x}_t),y_t)\vert I_1,\ldots,I_{t-1}\right].
    $$
    The conditional expectation can not be canceled
    since $\ell(f_{t,i}(\mathbf{x}_t),y_t)$
    dependents on $\{I_{\tau}\}^{t-1}_{\tau=1}$.
    For each $i\in[K]$,
    let $f$ be the optimal hypothesis in $\mathbb{H}_i$.
    Then we recovery the first statement of Theorem 2.\\
    Next we prove the second statement of Theorem 2.
    The proof is same with the proof of the first statement.
    The difference is that we use the fact
    $\vert\ell'(f(\mathbf{x}),y)\vert^2 \leq  C_0\ell(f(\mathbf{x}),y)$.
    We just need to reanalyze $\Xi_2$.
    \begin{align*}
        \Xi_2
        \leq\frac{3U^2}{2\lambda_{T,i}}
        +\frac{KC_0}{2}
        \sum^T_{t=1}\frac{\lambda_{t,i}}{\delta_t}
        \frac{\ell(f_{t,i}(\mathbf{x}_{t}),y_{t})}{p_{t,i}}\mathbb{I}_{I_t=i}.
    \end{align*}
    Let $G=1$ in $\delta_t$ and $\lambda_{t,i}$. $\eta_t$ keeps unchanged.
    Then the final regret satisfies
    \begin{align*}
        &\sum^T_{t=1}\ell(f_{t,I_t}(\mathbf{x}_t),y_t)
        -\sum^T_{t=1}\frac{\ell(f(\mathbf{x}_t),y_t)}{p_{t,i}}\mathbb{I}_{I_t=i}\\
        =&O\left(U^{\frac{2}{3}}C^{\frac{1}{3}}_0K^{-\frac{1}{3}}\tilde{L}^{\frac{2}{3}}_{1:K}(f)
        +U^{\frac{2}{3}}C^{\frac{1}{3}}_0K^{\frac{1}{6}}\tilde{L}_{1:K}^{\frac{1}{6}}(f)
        \sqrt{\sum^T_{t=1}\frac{\ell(f(\mathbf{x}_t),y_t)}{p_{t,i}}\mathbb{I}_{I_t=i}}\right).
    \end{align*}
    Taking expectation w.r.t. $\{I_t\}^T_{t=1}$ concludes the proof.

\section{Property of OMD with the Tsallis entropy regularizer}

    Let $\psi_{t}(\mathbf{p})=\sum^K_{i=1}\frac{-\alpha}{\eta_{t,i}}p^{\frac{1}{\alpha}}_{i}$,
    where $\alpha>1$.
    The gradient w.r.t. $\mathbf{p}$ is as follows
    $$
        \forall i\in[K],\quad\nabla_{i}\psi_{t}(\mathbf{p})=-\frac{1}{\eta_{t,i}}p^{\frac{1}{\alpha}-1}_i.
    $$
    $\forall \mathbf{u},\mathbf{v}\in \Delta_{K}$,
    the Bregman divergence defined by the $\alpha$-Tsallis entropy is
    \begin{align*}
        \mathcal{D}_{\psi_t}(\mathbf{u},\mathbf{v})
        =&\psi_t(\mathbf{u})-\psi_t(\mathbf{v})
        -\langle\nabla\psi_t(\mathbf{v}),\mathbf{u}-\mathbf{v}\rangle\\
        =&-\sum^K_{i=1}\frac{\alpha}{\eta_{t,i}}u^{\frac{1}{\alpha}}_i+
        \sum^K_{i=1}\frac{\alpha}{\eta_{t,i}}{v}^{\frac{1}{\alpha}}_i
        + \sum^K_{i=1}\frac{1}{\eta_{t,i}}{v}^{\frac{1}{\alpha}-1}_i(u_i-v_i)\\
        =&\sum^K_{i=1}\frac{1}{\eta_{t,i}}
        \left(-\alpha u^{\frac{1}{\alpha}}_i+
        (\alpha-1)v^{\frac{1}{\alpha}}_i+v^{\frac{1}{\alpha}-1}_iu_i\right).
    \end{align*}
    The first mirror updating is as follows,
    \begin{equation}
    \label{eq:ECML2022:first_mirror_updating}
        -\frac{1}{\eta_{t,i}}(q'_{t+1,i})^{\frac{1}{\alpha}-1}
        =-\frac{1}{\eta_{t,i}}q^{\frac{1}{\alpha}-1}_{t,i}-\tilde{c}_{t,i}.
    \end{equation}
    To derive the second mirror updating,
    we first write the Lagrange function.
    $$
        \mathcal{L}=
        \sum^K_{i=1}\frac{(\alpha-1)(q'_{t+1,i})^{\frac{1}{\alpha}}+(q'_{t+1,i})^{\frac{1}{\alpha}-1}q_i
        -\alpha q^{\frac{1}{\alpha}}_i}{\eta_{t,i}}
        -\mu\left(\sum^K_{i=1}q_i-1\right)
        -\sum^K_{i=1}\beta_iq_i.
    $$
    The optimal solution is $\mathbf{q}_{t+1}$, $\mu^\ast$ and $\{\beta^\ast_i\}^K_{i=1}$.
    The KKT condition is as follows
    $$
        q^{\frac{1}{\alpha}-1}_{t+1,i}=(q'_{t+1,i})^{\frac{1}{\alpha}-1}+\eta_{t,i}(-\mu^\ast-\beta^\ast_i)
        ,\quad
        \mathrm{s.t.}~\sum^K_{i=1}q_{t+1,i}=1,\beta^\ast_i\geq 0,\beta^\ast_iq_{t+1,i}=0.
    $$
    Combining with \eqref{eq:ECML2022:first_mirror_updating}, we obtain
    $$
        q_{t+1,i}=
        \left(q^{\frac{1}{\alpha}-1}_{t,i}+\eta_{t,i}(\tilde{c}_{t,i}-\mu^\ast-\beta^\ast_i)\right)^
        {\frac{\alpha}{1-\alpha}}.
    $$
    Next we construct a solution.
    Assuming that $\{\beta^\ast_i=0\}^K_{i=1}$.
    Then it must be $q_{t+1,i}>0, \forall i\in[K]$.
    Under this assumption,
    denote by
    $$
        f(\mu)=\sum^K_{i=1}\left(q^{\frac{1}{\alpha}-1}_{t,i}+\eta_{t,i}(\tilde{c}_{t,i}-\mu)\right)^
        {\frac{\alpha}{1-\alpha}},\alpha>1.
    $$
    We consider $\mu\in[0,\max_i\tilde{c}_{t,i}]$.
    Recalling that if $p_{t,I_t}>\max_i\eta_{t,i}$, then $\tilde{c}_{t,i}=\frac{c_{t,i}}{p_{t,i}}\mathbb{I}_{i=I_t}$.
    In this case, $\eta_{t,i}\mu \leq \frac{\eta_{t,i}}{p_{t,I_t}}\leq 1$.
    Thus
    $q^{\frac{1}{\alpha}-1}_{t,i}-\eta_{t,i}\mu>0$.
    If $p_{t,I_t}>\max_i\eta_{t,i}$,
    then $\tilde{c}_{t,i}=\frac{c_{t,i}}{p_{t,i}+\max_i\eta_{t,i}}\mathbb{I}_{i=I_t}$.
    Similar analysis gives $q^{\frac{1}{\alpha}-1}_{t,i}-\eta_{t,i}\mu>0$.
    It can be verified that $\frac{\mathrm{d}\,f(\mu)}{\mathrm{d}\,\mu}>0$
    in $[0,\max_i\tilde{c}_{t,i}]$.
    Thus $f(\mu)$ is increasing in $[0,\max_i\tilde{c}_{t,i}]$.
    Since $f(0)\leq 1$ and $f(\max_i\tilde{c}_{t,i})\geq 1$.
    There must be a $\mu^\ast\in[0,\max_i\tilde{c}_{t,i}]$ such that $f(\mu^\ast)=1$.
    We can solve $\mu^\ast$ by binary search.
    In this case,
    $q^{\frac{1}{\alpha}-1}_{t+1,i}>0$ for all $i\in[K]$.
    Thus the assumption $\beta^\ast_i=0,\forall i\in[K]$ holds.

\section{Proof of Theorem \ref{coro:IJCAI2022:IOKS:regret_bound}}

\begin{proof}
    Similar to the proof of Theorem 2,
    we decompose the regret as follows,
    \begin{align*}
        &\sum^T_{t=1}\ell(f_{t,I_t}(\mathbf{x}_t),y_t)
        -\sum^T_{t=1}\frac{\ell(f(\mathbf{x}_t),y_t)}{p_{t,i}}\mathbb{I}_{I_t=i}\\
        =&\ell_{\max}
        \underbrace{\sum^T_{t=1}\langle\mathbf{p}_t-\mathbf{u},\bar{c}_t\rangle}_{\Xi_1}+
        \underbrace{\sum^T_{t=1}\frac{\ell(f_{t,i}(\mathbf{x}_t),y_t)
        -\ell(f(\mathbf{x}_t),y_t)}{p_{t,i}}\mathbb{I}_{I_t=i}}_{\Xi_2},
    \end{align*}
    where $\mathbf{u}\in\Delta_{K}$
    and $\bar{c}_{t,i}=\frac{\ell(f_{t,i}(\mathbf{x}_t),y_t)}{p_{t,i}\ell_{\max}}\mathbb{I}_{i=I_t}
    =\frac{c_{t,i}}{p_{t,i}}\mathbb{I}_{i=I_t}$.
    We can let $\mathbf{u}=\mathbf{e}_i$.

\subsection{Analyzing $\Xi_1$}

    Recalling that $\mathbf{p}_t=(1-\delta)\mathbf{q}_t+\delta\frac{\mathbf{1}_K}{K}$
    where $\delta=T^{-\frac{3}{4}}$.
    Then $\Xi_1$ can be bounded as follows,
    \begin{align*}
        \Xi_1=&\sum^T_{t=1}\langle \mathbf{q}_t-\mathbf{u},\tilde{c}_t\rangle
        +\sum^T_{t=1}\langle \mathbf{p}_t-\mathbf{q}_t,\tilde{c}_t\rangle
        +\sum^T_{t=1}\langle \mathbf{p}_t-\mathbf{u},\bar{c}_t-\tilde{c}_t\rangle.
    \end{align*}
        $\langle \mathbf{q}_t-\mathbf{u},\tilde{c}_t\rangle$ is still decomposed as follows,
        $$
            \langle \mathbf{q}_t-\mathbf{u},\tilde{c}_t\rangle
            \leq\underbrace{\mathcal{D}_{\psi_t}(\mathbf{u},\mathbf{q}_t)
            -\mathcal{D}_{\psi_t}(\mathbf{u},\mathbf{q}_{t+1})}_{\Xi_{1,1}}
            \underbrace{-\mathcal{D}_{\psi_t}(\mathbf{q}'_{t+1},\mathbf{q}_t)+\langle \mathbf{q}_t-\mathbf{q}'_{t+1},\tilde{c}_t\rangle}_{\Xi_{1,2}}.
        $$
        We first analyze $\Xi_{1,1}$.\\
        Since $\mathbf{q}_1=\frac{1}{K}\mathbf{1}_K$, we have
        \begin{align*}
            \sum^T_{t=1}&\Xi_{1,1}
            \leq\mathcal{D}_{\psi_1}(\mathbf{u},\mathbf{q}_1)
            +\sum^T_{t=1}[\mathcal{D}_{\psi_{t+1}}(\mathbf{u},\mathbf{q}_{t+1})
            -\mathcal{D}_{\psi_t}(\mathbf{u},\mathbf{q}_{t+1})]\\
            =&\sum^K_{i=1}\frac{
            (\alpha-1)K^{-\frac{1}{\alpha}}+K^{1-\frac{1}{\alpha}}u_i-\alpha u^{\frac{1}{\alpha}}_i}{\eta_{1,i}}
            +\sum^T_{t=1}[\mathcal{D}_{\psi_{t+1}}(\mathbf{u},\mathbf{q}_{t+1})
            -\mathcal{D}_{\psi_t}(\mathbf{u},\mathbf{q}_{t+1})]\\
            \leq&\frac{\alpha}{\eta} K^{1-\frac{1}{\alpha}}
            +\sum^T_{t=1}\sum^K_{i=1}\left(\frac{1}{\eta_{t+1,i}}-\frac{1}{\eta_{t,i}}\right)
            h(u_i,q_{t+1,i}),
        \end{align*}
        where
        $$
            h(u_i,q_{t+1,i})=-\alpha u^{\frac{1}{\alpha}}_i+
            (\alpha-1)q^{\frac{1}{\alpha}}_{t+1,i}+u_iq^{\frac{1}{\alpha}-1}_{t+1,i}
            =\frac{u_i}{q^{\frac{7}{8}}_{t+1,i}}+7q^{\frac{1}{8}}_{t+1,i}-8u^{\frac{1}{8}}_i,
        $$
        and $\alpha=8$.\\
         Let $t_1,\ldots,t_{n_i}$ be the rounds such that
        $\eta_{t+1,i}=\upsilon\eta_{t,i}$.
        For $t\leq t_{n_i}$,
        we have $\eta_{t+1,i}\geq \eta_{t,i}$.
        For $t>t_{n_i}$,
        we have $\eta_{t+1,i}= \eta_{t,i}$.
        Following the proof of Lemma 13 in \cite{Agarwal2017Corralling},
        we directly obtain (i) $n_i\leq \frac{3}{4}\log_2{T}$; (ii) $\upsilon^{n_i}
        \leq \mathrm{e}^{\frac{3}{4}\log_2{T}\cdot\frac{2}{3\ln{T}}}=\mathrm{e}^{\frac{1}{2\ln{2}}}$;
        (iii) $\eta_{t_{n_i}+1,i}= \upsilon^{n_i}\eta_{t_1,i}\leq 2.1\eta_{1,i}$
        and (iv)
        $$
        \sum^T_{t=1}\sum^K_{i=1}\left(\frac{1}{\eta_{t+1,i}}-\frac{1}{\eta_{t,i}}\right)
            h(u_i,q_{t+1,i})
            \leq
            \frac{-\mathrm{e}^{-\frac{1}{2\ln{2}}}}{\eta_{1,i}\ln{T}}
            \left(\frac{u_i}{q^{\frac{7}{8}}_{t_{n_i}+1,i}}+7q^{\frac{1}{8}}_{t_{n_i}+1,i}-8u^{\frac{1}{8}}_i\right).
        $$
        It can be verified that $h(u_i,q_{t_{n_i}+1,i})$ with $u_i=1$
        increases with $1/q_{t_{n_i}+1,i}$ for any $0<q_{t_{n_i}+1,i} \leq 1$.
        Let $T\geq 40$. We further have
        $$
            \frac{1}{q_{t_{n_i}+1,i}}
            \geq \frac{1-\frac{1}{T^{\frac{3}{4}}}}{p_{t_{n_i}+1,i}}
            = \frac{1}{2}(1-\frac{1}{T^{\frac{3}{4}}})\rho_{t_{n_i}+1,i}
            \geq 0.468\rho_{T+1,i}.
        $$
        We further obtain
        $$
        \sum^T_{t=1}\Xi_{1,1}
            \leq\frac{8K^{\frac{7}{8}}}{\eta}
            -\frac{\left(0.468\rho_{T+1,i}\right)^{\frac{7}{8}}
            +7(\left(0.468\rho_{T+1,i}\right))^{-\frac{1}{8}}-8}{\mathrm{e}^{\frac{1}{2\ln{2}}}\eta_{1,i}\ln{T}}
           \lesssim\frac{8K^{\frac{7}{8}}}{\eta}-\frac{\rho^{\frac{7}{8}}_{T+1,i}}{4\eta\ln{T}},
        $$
        where $\eta_{1,i}=\eta$, $\forall i\in[K]$.
        Next we analyze $\Xi_{1,2}$.
        \begin{align*}
            &\Xi_{1,2}
            =\langle \mathbf{q}_{t}-\mathbf{q}'_{t+1},\tilde{c}_{t}\rangle
            -\sum^K_{i=1}\frac{1}{\eta_{t,i}}
            \left(-\alpha (q'_{t+1,i})^{\frac{1}{\alpha}}+
            (\alpha-1)q^{\frac{1}{\alpha}}_{t,i}+q'_{t+1,i}q^{\frac{1}{\alpha}-1}_{t,i}\right)\\
            =&\langle \mathbf{q}_{t}-\mathbf{q}'_{t+1},\tilde{c}_{t}\rangle
            -\sum^K_{i=1}\frac{-\alpha (q'_{t+1,i})^{\frac{1}{\alpha}}+
            (\alpha-1)q^{\frac{1}{\alpha}}_{t,i}+(q'_{t+1,i})^{\frac{1}{\alpha}}-\eta_{t,i}q'_{t+1,i}\tilde{c}_{t,i}}{\eta_{t,i}}\\
            =&\langle \mathbf{q}_{t},\tilde{c}_{t}\rangle
            +\sum^K_{i=1}\frac{\alpha-1}{\eta_{t,i}}
            \left((q'_{t+1,i})^{\frac{1}{\alpha}}-q^{\frac{1}{\alpha}}_{t,i}\right)\\
            =&\langle \mathbf{q}_{t},\tilde{c}_{t}\rangle
            +\sum^K_{i=1}\frac{\alpha-1}{\eta_{t,i}}
            \left(\left(\eta_{t,i}\tilde{c}_{t,i}+q^{\frac{1}{\alpha}-1}_{t,i}\right)^{\frac{1}{1-\alpha}}
            -q^{\frac{1}{\alpha}}_{t,i}\right)\\
            =&\langle \mathbf{q}_{t},\tilde{c}_{t}\rangle
            +\sum^K_{i=1}\frac{\alpha-1}{\eta_{t,i}}
            \left(q^{\frac{1}{\alpha}}_{t,i}\left(1-(-\eta_{t,i}q^{1-\frac{1}{\alpha}}_{t,i}\tilde{c}_{t,i})\right)^
            {-\frac{1}{\alpha-1}}
            -q^{\frac{1}{\alpha}}_{t,i}\right)\\
            \leq&\langle \mathbf{q}_{t},\tilde{c}_{t}\rangle
            +\sum^K_{i=1}\frac{\alpha-1}{\eta_{t,i}}
            \left(q^{\frac{1}{\alpha}}_{t,i}\left(1-\frac{\eta_{t,i}q^{1-\frac{1}{\alpha}}_{t,i}\tilde{c}_{t,i}}{\alpha-1}
            +\frac{\alpha}{2(\alpha-1)^2}\eta^2_{t,i}q^{2-\frac{2}{\alpha}}_{t,i}\tilde{c}^2_{t,i}\right)
            -q^{\frac{1}{\alpha}}_{t,i}\right)\\
            =&\sum^K_{i=1}\frac{\alpha}{2(\alpha-1)}\eta_{t,i}q^{2-\frac{1}{\alpha}}_{t,i}\tilde{c}^2_{t,i}
            \leq\frac{2.1\times 0.94^{2-\frac{1}{\alpha}}\alpha}{2(\alpha-1)}
            \eta p^{-\frac{1}{\alpha}}_{t,I_t}c^2_{t,I_t}
            \leq 1.1\eta p^{-\frac{1}{8}}_{t,I_t}c^2_{t,I_t},
        \end{align*}
        where the first inequality comes from
        Lemma \ref{lemma:ECML2022:basic_lemma_Tsallis_regularizer},
        and the last but one comes from $p_{t,i}\geq (1-\delta)q_{t,i}=(1-T^{-3/4})q_{t,i}$
        and $\eta_{t,i}\leq 2.1\eta$, $\forall i\in[K], t\in[T]$.\\
        Combining the upper bounds on $\Xi_{1,1}$ and $\Xi_{1,2}$ gives
        \begin{align*}
            \Xi_1
            \leq&\frac{8}{\eta} K^{\frac{7}{8}}-\frac{\rho^{\frac{7}{8}}_{T+1,i}}{4\eta\ln{T}}
            +1.1\eta p^{-\frac{1}{8}}_{t,I_t}c^2_{t,I_t}
            +\sum^T_{t=1}\langle \mathbf{p}_t-\mathbf{q}_t,\tilde{c}_t\rangle
            +\sum^T_{t=1}\langle \mathbf{p}_t-\mathbf{u},\bar{c}_t-\tilde{c}_t\rangle\\
            \leq&\frac{8}{\eta} K^{\frac{7}{8}}-\frac{\rho^{\frac{7}{8}}_{T+1,i}}{4\eta\ln{T}}
            +1.1\eta p^{-\frac{1}{8}}_{t,I_t}c^2_{t,I_t}
            +\frac{T^{-\frac{3}{4}}}{K}\sum^T_{t=1}\sum^K_{i=1}\tilde{c}_{t,i}
            +\sum^T_{t=1}\frac{\max_i\eta_{t,i}c_{t,I_t}\xi_t}{p_{t,I_t}+\max_i\eta_{t,i}}
        \end{align*}
        where $\xi_t:=\mathbb{I}_{p_{t,I_t}<\max_i\eta_{t,i}}$.

    \begin{lemma}[\cite{Zimmert2019An}]
    \label{lemma:ECML2022:basic_lemma_Tsallis_regularizer}
        For any $x<1$, $y>0$,
        $$
            (1-x)^{-y}\leq 1+yx+
            \left\{
            \begin{array}{ll}
            \frac{y(1+y)}{2}x^2(1-x)^{-\max\{y,1\}}, x\geq 0,\\
            \frac{y(1+y)}{2}x^2,x<0.
            \end{array}
            \right.
        $$
    \end{lemma}

\subsection{Analyzing $\Xi_2$}

    The analysis is same with the analysis of $\Xi_2$ in the proof of Theorem 1
    (see Section \ref{sec:ECML2022_supp:theorem_1:xi_2}).
    Let $\lambda_{t,i}=\frac{U}
    {\sqrt{2}\cdot\sqrt{1+\sum^t_{\tau=1}\Vert\tilde{\nabla}_{\tau,i}\Vert^2_{\mathcal{H}_i}}}$.
    \begin{align*}
        \Xi_2
        \leq\frac{3U^2}{2\lambda_{T,i}}
        +\frac{1}{2}\sum^T_{t=1}\lambda_{t,i}\Vert\tilde{\nabla}_{t,i}\Vert^2_{\mathcal{H}_{i}}
        \leq&2\sqrt{2}U\sqrt{\sum^T_{\tau=1}\Vert\tilde{\nabla}_{\tau,i}\Vert^2_{\mathcal{H}_i}}.
    \end{align*}

\subsection{Combining $\Xi_1$ with $\Xi_2$}

    Combining the upper bounds on $\Xi_1$ and $\Xi_2$ yields
    \begin{align*}
        &\sum^T_{t=1}\ell(f_{t,I_t}(\mathbf{x}_t),y_t)
        -\sum^T_{t=1}\frac{\ell(f(\mathbf{x}_t),y_t)}{p_{t,i}}\mathbb{I}_{I_t=i}
        \leq \ell_{\max}\left(-\frac{\rho^{\frac{7}{8}}_{T+1,i}}{4\eta\ln{T}}
        +1.1\eta \sum^T_{t=1}p^{-\frac{1}{8}}_{t,I_t}c^2_{t,I_t}\right.\\
        &\left.+\frac{8}{\eta} K^{\frac{7}{8}}
        +\frac{T^{-\frac{3}{4}}}{K}\sum^T_{t=1}\sum^K_{i=1}\tilde{c}_{t,i}
            +\sum^T_{t=1}\frac{\max_i\eta_{t,i}c_{t,I_t}\xi_t}{p_{t,I_t}+\max_i\eta_{t,i}}\right)
        +2\sqrt{2}U\sqrt{\sum^T_{t=1}\Vert\tilde{\nabla}_{t,i}\Vert^2_{\mathcal{H}_i}}.
    \end{align*}
    where
    \begin{align*}
        \sqrt{\sum^T_{t=1}\Vert\tilde{\nabla}_{t,i}\Vert^2_{\mathcal{H}_i}}
        \leq G_1\sqrt{\sum^T_{t=1}\frac{1}{p^2_{t,i}}\mathbb{I}_{I_t=i}}
        \leq G_1\sqrt{\max_t\frac{1}{p_{t,i}}\cdot\tilde{\nabla}_{1:T}}
        =G_1\sqrt{\rho_{T+1,i}\cdot\tilde{\nabla}_{1:T}},
    \end{align*}
    where $\tilde{\nabla}_{1:T}=\sum^T_{t=1}\frac{1}{p_{t,i}}\mathbb{I}_{I_t=i}$.
    Next we bound the following term
    $$
        \max_{\rho_{T+1,i}\geq 0}g(\rho_{T+1,i}):=2\sqrt{2}UG_1\sqrt{\rho_{T+1,i}\cdot\tilde{\nabla}_{1:T}}
        -\frac{\ell_{\max}\rho^{\frac{7}{8}}_{T+1,i}}{4\eta\ln{T}}.
    $$
    Let $\eta=\frac{2\ell_{\max}K^{\frac{3}{8}}}{UG_1\sqrt{T}\ln^{\frac{1}{2}}(T)}$.
    Let the first derivative of $g(\rho_{T+1,i})$ equals to zero.
    The maximum of $g(\rho_{T+1,i})$ is obtained at $\sqrt{\rho_{T+1,i}}
    =(\frac{48\sqrt{2}}{7})^{\frac{4}{3}}(\tilde{\nabla}_{1:T}\ln{(T)}/T)^{\frac{2}{3}}\sqrt{K}$.
    $$
        \max_{\rho_{T+1,i}\geq 0}g(\rho_{T+1,i})=
        25.1UG_1\tilde{\nabla}_{1:T}^{\frac{7}{6}}T^{-\frac{2}{3}}\sqrt{K}\ln^{\frac{2}{3}}{T}.
    $$
    Combining all results gives
    \begin{align*}
        &\sum^T_{t=1}\ell(f_{t,I_t}(\mathbf{x}_t),y_t)
        -\sum^T_{t=1}\frac{\ell(f(\mathbf{x}_t),y_t)}{p_{t,i}}\mathbb{I}_{I_t=i}
        \leq \ell_{\max}\left(1.1\eta \sum^T_{t=1}p^{-\frac{1}{8}}_{t,I_t}c^2_{t,I_t}\right.\\
        &\left.+\frac{8}{\eta} K^{\frac{7}{8}}
        +\frac{T^{-\frac{3}{4}}}{K}\sum^T_{t=1}\sum^K_{i=1}\tilde{c}_{t,i}
            +\sum^T_{t=1}\frac{\max_i\eta_{t,i}c_{t,I_t}\xi_t}{p_{t,I_t}+\max_i\eta_{t,i}}\right)
            +\frac{25.1UG_1}{T^{\frac{2}{3}}}\tilde{\nabla}_{1:T}^{\frac{7}{6}}\sqrt{K}\ln^{\frac{2}{3}}{T}.
    \end{align*}
    Next we need to take expectation w.r.t. $\{I_t\}^T_{t=1}$.
    Note that $\mathbb{E}[\frac{\ell(f(\mathbf{x}_t),y_t)}{p_{t,i}}\mathbb{I}_{I_t=i}]
    =\mathbb{E}_{1:t-1}[\ell(f(\mathbf{x}_t),y_t)]=\ell(f(\mathbf{x}_t),y_t)$.
    The condition is canceled since $\ell(f(\mathbf{x}_t),y_t)$ is independent of $\{I_{\tau}\}^{t-1}_{\tau=1}$.
    Similarly,
    we have $\mathbb{E}[\tilde{c}_{t,i}]\leq 1$ and $\mathbb{E}[c_{t,I_t}]\leq 1$.
    (i)
        $$
            \mathbb{E}\left[p^{-\frac{1}{8}}_{t,I_t}c^2_{t,I_t}\right]
            \leq \sum^K_{i=1}p^{1-\frac{1}{8}}_{t,i}
            =K^{\frac{1}{8}},
        $$
        where the maximum is obtained at $p_{t,i}=\frac{1}{K}$ for all $i\in[K]$.
    (ii)
    \begin{align*}
        \mathbb{E}\left[\sum^T_{t=1}\frac{\max_i\eta_{t,i}c_{t,I_t}\xi_t}{p_{t,I_t}+\max_i\eta_{t,i}}\right]
        =&\mathbb{E}\left[\sum^T_{t=1}\sum^K_{i=1}
        \frac{(\max_i\eta_{t,i})c_{t,i}\mathbb{I}_{i=I_t}\mathbb{I}_{p_{t,I_t}<\max_i\eta_{t,i}}}{p_{t,i}
        +\max_i\eta_{t,i}}\right]\\
        \leq& \mathbb{E}\left[\sum^T_{t=1}\sum^K_{i=1}(K-1)(\max_i\eta_{t,i})^2
        \frac{p_{t,i}c_{t,i}}{p_{t,i}+\max_i\eta_{t,i}}\right]\\
        \leq&4.5K^2\eta^2T=\frac{11\ell^2_{\max}K^{\frac{11}{4}}}{U^2G^2_1\ln{T}},
    \end{align*}
    where $\max_i\eta_{t,i}\leq 2.1\eta$
    and $\eta=\frac{3\ell_{\max}K^{\frac{3}{8}}}{2UG_1\sqrt{T\ln{T}}}$.
    (iii)
    $$
        \mathbb{E}\left[\tilde{\nabla}_{1:T}^{\frac{7}{6}}\right]
        \leq \left(\mathbb{E}\left[\tilde{\nabla}_{1:T}\right]\right)^{\frac{7}{6}}
        =\left(\mathbb{E}\left[\sum^T_{t=1}\frac{1}{p_{t,i}}\mathbb{I}_{I_t=i}\right]\right)^{\frac{7}{6}}
        \leq T^{\frac{7}{6}}.
    $$
    Thus the final expected regret bound satisfies
    \begin{align*}
        &\mathbb{E}\left[\sum^T_{t=1}\ell(f_{t,I_t}(\mathbf{x}_t),y_t)\right]
        -\sum^T_{t=1}\ell(f(\mathbf{x}_t),y_t)\\
        \leq& 30UG_1\sqrt{KT}\ln^{\frac{2}{3}}{T}
        +\frac{11\ell^3_{\max}K^{\frac{11}{4}}}{U^2G^2_1\ln{T}}
        +\frac{1.7\ell^2_{\max}\sqrt{TK}}{UG_1\sqrt{\ln{T}}} +\ell_{\max}T^{\frac{1}{4}}.
    \end{align*}
    Omitting the lower order terms concludes the proof.
\end{proof}

\section{Proof of Theorem \ref{thm:ECML2022:OKS:loss_bound:RF_OKS++}}

    \begin{assumption}
    \label{ass:IJCAI2022:bounded_feature_mapping}
        For a fixed $\kappa_i\in\mathcal{K}$,
        there is a bounded constant $B_i$ such that,
        $\forall \mathbf{x}\in\mathcal{X}$,
        $\vert \phi_{\kappa_i}(\mathbf{x},\omega)\vert\leq B_i$.
    \end{assumption}

    We first establish a technique lemma.
    \begin{lemma}
    \label{lemma:IJCAI2022:approximate_error}
        For a fixed $\kappa_i$,
        with probability at least $1-\delta$,
        $\forall f\in\mathbb{H}_i$,
        there is a $\hat{f}\in\mathcal{F}_i$
        such that $\vert f(\mathbf{x})-\hat{f}(\mathbf{x})\vert
        \leq \frac{UB_i}{\sqrt{D_i}}\sqrt{2\ln\frac{1}{\delta}}$.
    \end{lemma}

    \begin{proof}
        For a $f\in\mathbb{H}_i$,
        let $f(\mathbf{x})=\int_{\Omega}\alpha(\omega)\phi_{\kappa_i}(\mathbf{x},\omega)p_{\kappa_i}(\omega)
        \mathrm{d}\,\omega$,
        where $\vert \alpha(\omega)\vert \leq U$.
        We construct $\hat{f}(\mathbf{x})=\sum^{D_i}_{j=1}
        \frac{\alpha(\omega_j)}{D_i}\phi_{\kappa_i}(\mathbf{x},\omega_j)$.
        It can be verified that
        \begin{align*}
            \forall j\in[D_i],\quad\mathbb{E}_{\omega_j}[\hat{f}_j(\mathbf{x})]
            =\mathbb{E}\left[\alpha(\omega_j)\phi_{\kappa_i}(\mathbf{x},\omega_j)\right]
            =&\int_{\Omega}\alpha(\omega_j)p_{\kappa}(\omega_j)
            \phi_{\kappa_i}(\mathbf{x},\omega_j)\mathrm{d}\,\omega_j\\
            =&f(\mathbf{x}),
        \end{align*}
        where $\hat{f}_j(\mathbf{x})=\alpha(\omega_j)\phi_{\kappa_i}(\mathbf{x},\omega_j)$.
        Besides, $\vert \frac{1}{D_i}\alpha(\omega_j)\vert\leq \frac{U}{D_i}$.
        Under the condition of Assumption \ref{ass:IJCAI2022:bounded_feature_mapping},
        we have $\vert \hat{f}_j(\mathbf{x})\vert \leq
        \vert\alpha(\omega_j)\vert\cdot \vert\phi_{\kappa}(\mathbf{x},\omega_j)\vert\leq UB_i$.
        Since $\hat{f}_{1}(\mathbf{x}),\ldots,\hat{f}_{D}(\mathbf{x})$
        are i.i.d.,
        then the Hoeffding's inequality yields
        $$
            \mathbb{P}\left[\vert \hat{f}(\mathbf{x})-f(\mathbf{x})\vert\geq \epsilon\right]
            \leq \exp\left(-\frac{D_i\epsilon^2}{2(UB_i)^2}\right).
        $$
        Let the fail probability equal to $\delta$,
        then we conclude the proof.
    \end{proof}

        Next we begin to prove Theorem \ref{thm:ECML2022:OKS:loss_bound:RF_OKS++}.\\
\begin{proof}[of Theorem \ref{thm:ECML2022:OKS:loss_bound:RF_OKS++}]
    $\forall i\in[K]$, $\forall f_i\in\mathbb{H}_i$,
    the expected regret w.r.t. $f_i$ can be decomposed as follows,
    $$
        \mathbb{E}\left[\bar{L}_{\mathbf{p}_{1:T}}\right]-L_T(f)
        =\underbrace{\mathbb{E}\left[\bar{L}_{\mathbf{p}_{1:T}}\right]-L_T(\hat{f}_i)}_{\Xi_1}
        +\underbrace{L_T(\hat{f}_i)-L_T(f_i)}_{\Xi_2},
    $$
    where $\hat{f}_i\in\mathcal{F}_i$.
    Denote by $\hat{f}_i=\sum^{D_i}_{j=1}\hat{\alpha}_j\phi_i(\cdot,\omega_j)
    =\mathbf{w}^\top z_i(\cdot)$
    where ${\bf w}=\sqrt{D_i}\hat{{\bm \alpha}}$ and $\Vert\mathbf{w}\Vert_2\leq U$.
    At each round $t$,
    our algorithms maintain $\{\mathbf{w}^i_t\}^K_{i=1}$
    and ensure $\Vert\mathbf{w}^i_t\Vert_2\leq U$.
    Thus the upper bound of $\Xi_1$ is given in Theorem 2 and Theorem 3.
    We just need to bound $\Xi_2$.

\subsection{The first statement}

        Recalling that Theorem \ref{thm:ECML2022:OKS:loss_bound}.
        Let $\bar{L}_T=\sum^K_{j=1}L_T(\hat{f}_j)$.
        We obtain
        $$
            \Xi_1:=\mathbb{E}\left[\bar{L}_{\mathbf{p}_{1:T}}\right]- L_T(\hat{f}_i)
            =O\left(U^{\frac{2}{3}}C^{\frac{1}{3}}_0K^{-\frac{1}{3}}
            \bar{L}^{\frac{2}{3}}_T
            +U^{\frac{2}{3}}C^{\frac{1}{3}}_0K^{\frac{1}{6}}\bar{L}^{\frac{1}{6}}_T
            L^{\frac{1}{2}}_T(\hat{f}_i)\right).
        $$
        The key is to bound $\Xi_2$.
        Using Lemma \ref{lemma:IJCAI2022:approximate_error},
        with probability at least $1-T\delta$,
        \begin{align}
            \Xi_2:=L_T(\hat{f}_i)-L_T(f_i)
            \leq&\sum^T_{t=1}\ell'(\hat{f}_i(\mathbf{x}_t),y_t)
            \cdot(\hat{f}_i(\mathbf{x}_t)-f_i(\mathbf{x}_t))\nonumber\\
            \leq&
            \sum^T_{t=1}\vert\ell'(\hat{f}_i(\mathbf{x}_t),y_t)\vert\epsilon
            \leq C_0L_T(\hat{f}_i)\cdot \epsilon,
            \label{eq:AAAI2022:analysis_first_result_intermediate_step_1}
        \end{align}
        where $\epsilon = \frac{UB_i}{\sqrt{D_i}}\sqrt{2\ln\frac{1}{\delta}}$.
        The fail probability $T\delta$
        comes from the union-of-events bound to $t=1,\ldots,T$.
        Rearranging terms yields
        \begin{align*}
            &L_T(\hat{f}_i)\leq\frac{1}{(1-C_0\epsilon)}L_T(f_i)<4L_T(f_i),\\
            &L_T(\hat{f}_i)-L_T(f_i) \leq \frac{C_0\epsilon}{(1-C_0\epsilon)}L_T(f_i)
            <4C_0\epsilon L_T(f_i),
        \end{align*}
        where we require that $C_0\epsilon \in(0,\frac{3}{4})$.
        Substituting into $\epsilon =\frac{UB_i}{\sqrt{D_i}}\sqrt{2\ln\frac{1}{\delta}}$.
        This condition becomes $D_i>\frac{32}{9}C^2_0U^2B^2_i\ln\frac{1}{\delta}$.
        Combining $\Xi_1$ and $\Xi_2$ yields, with probability at least $1-\delta$,
        \begin{align*}
            &\forall i\in[K],~\mathbb{E}\left[\bar{L}_{\mathbf{p}_{1:T}}\right]- L_T(f_i)\\
            =&O\left(U^{\frac{2}{3}}\left(\frac{GC_0}{K}\right)^{\frac{1}{3}}
            \left[\sum^K_{j=1}L_T(f_j)\right]^{\frac{2}{3}}
            +U^{\frac{2}{3}}(GC_0)^{\frac{1}{3}}K^{\frac{1}{6}}\left[\sum^K_{j=1}L_T(f_j)\right]^{\frac{1}{6}}
            L^{\frac{1}{2}}_T(f_i)\right)\\
            &+4C_0\frac{UB_i}{\sqrt{D_i}}\sqrt{2\ln\frac{KT}{\delta}} L_T(f_i).
        \end{align*}
        Let $f_i=f^\ast_i$ and $D_i=\beta_{\kappa_i}\mathcal{R}$.
        Then we recovery the first statement.

\subsection{The second statement}

        Using Lemma \ref{lemma:IJCAI2022:approximate_error},
        with probability at least $1-T\delta$,
        \begin{align*}
            \Xi_2:\leq&\sum^T_{t=1}\ell'(\hat{f}_i(\mathbf{x}_t),y_t)
            \cdot(\hat{f}_i(\mathbf{x}_t)-f_i(\mathbf{x}_t))\\
            \leq& \sqrt{T\cdot\sum^T_{t=1}\vert\ell'(\hat{f}_i(\mathbf{x}_t),y_t)\vert^2}\cdot \epsilon
            \leq \sqrt{C_0TL_T(\hat{f}_i)}\cdot \epsilon.
        \end{align*}
        Solving for $L_T(\hat{f}_i)$ yields
        \begin{align*}
            L_T(\hat{f}_i)-L_T(f_i)\leq& \sqrt{C_0TL_T(f_i)}\cdot \epsilon+C_0T\epsilon^2,\\
            L_T(\hat{f}_i) \leq&\frac{3}{2}(L_T(f_i)+C_0T\epsilon^2).
        \end{align*}
        Recalling that $\epsilon =\frac{UB_i}{\sqrt{D_i}}\sqrt{2\ln\frac{1}{\delta}}$.\\
        Combining $\Xi_1$ and $\Xi_2$ yields, with probability at least $1-\delta$,
        \begin{align*}
            &\forall i\in[K],~\mathbb{E}\left[\bar{L}_{\mathbf{p}_{1:T}}\right]- L_T(f_i)\\
            =&O\left(U^{\frac{2}{3}}\left(\frac{GC_0}{K}\right)^{\frac{1}{3}}
            \left[\sum^K_{j=1}L_T(f_j)\right]^{\frac{2}{3}}
            +U^{\frac{2}{3}}(GC_0)^{\frac{1}{3}}K^{\frac{1}{6}}\left[\sum^K_{j=1}L_T(f_j)\right]^{\frac{1}{6}}
            L^{\frac{1}{2}}_T(f_i)\right.\\
            &+\left.\frac{C_0U^2B^2_iT}{D_i}\ln\frac{KT}{\delta}
            +\frac{UB_i}{\sqrt{D_i}}\sqrt{C_0TL_T(f_i)\ln\frac{KT}{\delta}}\right).
        \end{align*}
        Let $f_i=f^\ast_i$ and $D_i=\beta_{\kappa_i}\mathcal{R}$.
        Then we recovery the second statement.
\end{proof}

\section{Proof of Theorem \ref{thm:IJCAI2022:RF-IOKS:regret_bound}}

    \begin{proof}
        Using Assumption \ref{ass:AAAI2022:lipschitz_condition},
        we change \eqref{eq:AAAI2022:analysis_first_result_intermediate_step_1}.
        With probability at least $1-T\delta$,
        $$
            L_T(\hat{f}_i)-L_T(f_i) \leq TG\epsilon.
        $$
        Combining $\Xi_1$ and $\Xi_2$ gives, with probability at least $1-KT\delta$,
        $\forall i\in[K]$,
        \begin{align*}
            \forall f\in\mathbb{H}_i,~&\mathbb{E}\left[L_T\right]-L_T(f_i)\\
            =&O\left(GU\sqrt{KT}\ln^{\frac{2}{3}}{T}+\frac{\ell^3_{\max}K^{\frac{11}{4}}}{U^2G^2_1\sqrt{\ln{T}}}
            +\frac{GB_iUT}{\sqrt{\beta_{\kappa_i}\mathcal{R}}}\sqrt{\ln{\frac{2}{\delta}}}\right).
        \end{align*}
        The fail probability $KT\delta$
        comes from the union-of-events bound to $t=1,\ldots,T$ and $i=1,\ldots,K$.
        Let $\delta'=\frac{\delta}{KT}$ and $f_i=f^\ast_i$.
        Then we conclude the proof.
    \end{proof}

\section{Omitted Algorithms}

    In this section,
    we give the pseudo-code of RF-OKS++
    and RF-IOKS.

    \begin{algorithm}[!t]
        \caption{RF-OKS++}
        \footnotesize
        \label{alg:IJCAI2022:RF-OKS++}
        \begin{algorithmic}[1]
        \Require {$\mathcal{K}=\{\kappa_i,D_i\}^K_{i=1}$, $\mathcal{R}$}
        \Ensure  $\{{\bf w}^i_{1}={\bf 0}\}^K_{i=1}$, $\mathbf{q}_1=\mathbf{p}_1=\frac{1}{K}\mathbf{1}_K$
        \State For all $i\in[K]$, sample $\{\omega_{i,j}\}^{D_i}_{j=1}$ from $p_{\kappa_i}(\omega)$
        independently
        \For{$t=1,\ldots,T$}
            \State Receive ${\bf x}_t$
            \State Sample a kernel $\kappa_{I_t}$ where $I_t\sim \mathbf{p}_t$
            \State Compute feature mapping $z_{I_t}(\mathbf{x}_t)
            =\frac{1}{\sqrt{D_{I_t}}}
            (\phi_{\kappa}(\mathbf{x}_t,\omega_{I_t,1}),\ldots,\phi_{\kappa}(\mathbf{x}_t,\omega_{I_t,D_{I_t}}))$
            \State Output $\hat{y}_t=z^\top_{I_t}(\mathbf{x}_t){\bf w}^{I_t}_{t}$ or $\mathrm{sign}(\hat{y}_t)$
                \State Compute $\lambda_{t,I_t}=\frac{U^{\frac{4}{3}}(\max\{GC_0U^2K^2,8\tilde{C}_{t,K}\})^{-\frac{1}{6}}}
        {\sqrt{4/3}K^{\frac{1}{6}}(GC_0)^{\frac{1}{3}}\sqrt{1+\sum^t_{\tau=1}
        \frac{\ell(f_{\tau,I_t}(\mathbf{x}_{\tau}),y_{\tau})}{p_{\tau,I_t}}}}$
                \State Update $
            {\bf w}^{I_t}_{t+1}=\mathop{\arg\min}_{\mathbf{w}\in\mathcal{F}_{I_t}}
            \left\Vert \mathbf{w}-\left({\bf w}^{I_t}_t
            -\lambda_{t,I_t}\nabla_{{\bf w}^{I_t}_t}
            \ell(\hat{f}_{t,I_t}(\mathbf{x}_t),y_t)\frac{1}{p_{t,I_t}}\right)\right\Vert^2_2$
                \State Construct $\tilde{c}_{t,I_t}=\frac{1}{p_{t,I_t}}c_{t,I_t}$
            \State Compute $\delta_t=\frac{1}{2}{(GC_0)^{\frac{1}{3}}(UK)^{\frac{2}{3}}}/
        {\max\left\{(GC_0)^{\frac{1}{3}}(UK)^{\frac{2}{3}},
        2\tilde{C}^{\frac{1}{3}}_{t,K}\right\}},
        \eta_t=\frac{\sqrt{2\ln{K}}}{\sqrt{1+\sum^t_{\tau=1}q_{\tau,I_{\tau}}\tilde{c}^2_{\tau,I_{\tau}}}}$
            \State Update $q_{t+1,i}=\frac{\exp(-\eta_t\sum^t_{\tau=1}\tilde{c}_{\tau,i})}
            {\sum^K_{j=1}\exp(-\eta_t\sum^t_{\tau=1}\tilde{c}_{\tau,j})}$
            and set $\mathbf{p}_{t+1}=(1-\delta_t)\mathbf{q}_{t+1}+\frac{\delta_t}{K}\mathbf{1}_K$;
        \EndFor
        \end{algorithmic}
    \end{algorithm}

    \begin{algorithm}[!t]
        \caption{RF-IOKS}
        \footnotesize
        \label{alg:IJCAI2022:RF-IOKS}
        \begin{algorithmic}[1]
        \Require {$\mathcal{K}=\{\kappa_i,D_i\}^K_{i=1}$, $\mathcal{R}$,
        $\alpha=8$, $\upsilon=\mathrm{e}^{\frac{2}{3\ln{T}}}$}
        \Ensure  $\{{\bf w}^i_{1}={\bf 0}\}^K_{i=1}$, $\mathbf{q}_1=\mathbf{p}_1=\frac{1}{K}\mathbf{1}_K$
        \State For all $i\in[K]$, sample $\{\omega_{i,j}\}^{D_i}_{j=1}$ from $p_{\kappa_i}(\omega)$
        independently
        \For{$t=1,\ldots,T$}
            \State Receive ${\bf x}_t$
            \State Sample a kernel $\kappa_{I_t}$ where $I_t\sim \mathbf{p}_t$
            \State Compute feature mapping $z_{I_t}(\mathbf{x}_t)$
            following line 5 in Algorithm \ref{alg:IJCAI2022:RF-OKS++}
            \State Output $\hat{y}_t=z^\top_{I_t}(\mathbf{x}_t){\bf w}^{I_t}_{t}$ or $\mathrm{sign}(\hat{y}_t)$
            \State Compute $\lambda_{t,I_t}={\sqrt{0.5}U}/
            {\sqrt{1+\sum^t_{\tau=1}\Vert\tilde{\nabla}_{\tau,I_t}\Vert^2_{\mathcal{H}_{I_t}}}}$
            \State Update ${\bf w}^{I_t}_{t+1}$ following line 8 in Algorithm \ref{alg:IJCAI2022:RF-OKS++}
            \State Construct $\tilde{c}_{t,I_t}=\frac{c_{t,I_t}}{p_{t,I_t}}$
                    if $p_{t,I_t}>\max_i\eta_{t,i}$, otherwise
                    $\tilde{c}_{t,I_t}=\frac{c_{t,I_t}}{p_{t,I_t}+\max_i\eta_{t,i}}$
            \State Compute $q_{t+1,i}=\left(q^{-\frac{7}{8}}_{t,i}+\eta_{t,i}(\tilde{c}_{t,i}-\mu^\ast)\right)^{-\frac{8}{7}}$
                    for all $i\in[K]$
            \State Compute $\mathbf{p}_{t+1}=(1-\delta)\mathbf{q}_{t+1}+\frac{\delta}{K}\mathbf{1}_K$
            \If{$\frac{1}{p_{t+1,i}}>\rho_{t,i}$}
                \State $\rho_{t+1,i}=\frac{2}{p_{t+1,i}}$, $\eta_{t+1,i}=\upsilon\eta_{t,i}$
            \Else
                \State $\rho_{t+1,i}=\rho_{t,i}$, $\eta_{t+1,i}=\eta_{t,i}$
            \EndIf
        \EndFor
        \end{algorithmic}
    \end{algorithm}

    In our experiments,
    we use the Gaussian kernel $\kappa(\mathbf{x},\mathbf{v})
    =\exp(-\frac{\Vert\mathbf{x}-\mathbf{v}\Vert^2_2}{2\sigma^2})$.
    The feature mapping can be computed as follows \cite{Rahimi2007Random},
    $$
        z_{\kappa}(\mathbf{x})
        =\frac{1}{\sqrt{D}}\left((\cos(\omega^\top_1\mathbf{x}),\sin(\omega^\top_1\mathbf{x})),
        \ldots,(\cos(\omega^\top_D\mathbf{x}),\sin(\omega^\top_D\mathbf{x}))\right),
    $$
    where $\{\sigma_i\}^D_{i=1}\sim\mathcal{N}(\mathbf{0}_d,\sigma^{-2}\mathbf{1}_{d\times d})$,
    i.e., a multi-dimensional normal distribution.

\section{More Experiments}

    In this section,
    we show more experiments for online kernel selection with time constraint.
    The baseline algorithms and parameter setting are same with Section 6.2 in the main text.
    The results are given in Table \ref{tab:Supp_ECML2022:with_constraint:logistic_loss},
    Table \ref{tab:Supp_ECML2022:with_constraint:square_loss}
    and Table \ref{tab:Supp_ECML2022:with_constraint:absolute_loss}.

    \begin{table}[!t]
      \centering
      \setlength{\tabcolsep}{1.7mm}
      \caption{Online kernel selection with time constraint in the regime of logistic loss}
      \label{tab:Supp_ECML2022:with_constraint:logistic_loss}
      \begin{tabular}{l|c|rr|c|rr}
        \Xhline{0.8pt}
        \multirow{2}{*}{Algorithm}&\multirow{2}{*}{B\text{-}D}&\multicolumn{2}{c|}{magic04}
        &\multirow{2}{*}{B\text{-}D}&\multicolumn{2}{c}{SUSY}  \\
        \cline{3-4}\cline{6-7}&&AMR (\%)&$t_p*10^5(s)$&&AMR (\%)&$t_p*10^5(s)$  \\
        \hline
        RF-OKS     & 500 & 22.45 $\pm$ 0.21 & 6.45  &500 & 33.78 $\pm$ 0.45 & 6.89\\
        LKMBooks   & 250 & 22.73 $\pm$ 0.96 & 6.67  &220 & 38.47 $\pm$ 0.77 & 6.89\\
        Raker      & 40  & 22.83 $\pm$ 0.70 & 6.80  & 40 & 33.62 $\pm$ 1.20 & 6.75\\
        \hline
        RF-IOKS    & 380 & 22.21 $\pm$ 0.26 & 6.59  &320 & 33.93 $\pm$ 0.27 & 6.94\\
        RF-OKS++   & 400 & \textbf{18.41} $\pm$ \textbf{0.35} & 6.28
                   & 400 & \textbf{28.93} $\pm$ \textbf{0.96} & 6.63 \\
        \Xhline{0.8pt}
      \end{tabular}
    \end{table}

    \begin{table}[!t]
      \centering
      \setlength{\tabcolsep}{1.7mm}
      \caption{Online kernel selection with time constraint in the regime of square loss}
      \label{tab:Supp_ECML2022:with_constraint:square_loss}
      \begin{tabular}{l|c|rr|c|rr}
        \Xhline{0.8pt}
        \multirow{2}{*}{Algorithm}&\multirow{2}{*}{B\text{-}D}&\multicolumn{2}{c|}{bank}
        &\multirow{2}{*}{B\text{-}D}&\multicolumn{2}{c}{ailerons}  \\
        \cline{3-4}\cline{6-7}&&$\mathrm{AL}*10^2$&$t_p*10^5(s)$&&$\mathrm{AL}*10^2$&$t_p*10^5(s)$ \\
        \hline
        RF-OKS     & 450 & 2.55 $\pm$ 0.03 & 7.57 & 480 & 2.26 $\pm$ 0.06 & 7.83\\
        LKMBooks   & 220 & 2.24 $\pm$ 0.02 & 7.52 & 200 & 1.72 $\pm$ 0.13 & 7.55\\
        Raker      & 50  & 2.38 $\pm$ 0.02 & 7.93 & 80  & \textbf{0.89} $\pm$ \textbf{0.06} & 7.72\\
        \hline
        RF-IOKS    & 370 & 2.63 $\pm$ 0.02 & 7.89 & 370 & 3.35 $\pm$ 0.04 & 7.87\\
        RF-OKS++   & 400 & \textbf{2.14} $\pm$ \textbf{0.05} & 7.30
                   & 400 & 1.68 $\pm$ 0.03 & 7.70\\
        \Xhline{0.8pt}
      \end{tabular}
    \end{table}

    \begin{table}[!t]
      \centering
      \setlength{\tabcolsep}{1.5mm}
      \caption{Online kernel selection with time constraint in the regime of absolute loss}
      \label{tab:Supp_ECML2022:with_constraint:absolute_loss}
      \begin{tabular}{l|c|rr|c|rr}
        \Xhline{0.8pt}
        \multirow{2}{*}{Algorithm}&\multirow{2}{*}{B\text{-}D}&\multicolumn{2}{c|}{bank}
        &\multirow{2}{*}{B\text{-}D}&\multicolumn{2}{c}{ailerons}  \\
        \cline{3-4}\cline{6-7}&&AL&$t_p*10^5$&&AL&$t_p*10^5$ \\
        \hline
        RF-OKS     &500 & 0.0969 $\pm$ 0.0004 & 8.63
                   &400 & 0.0728 $\pm$ 0.0005 & 8.52\\
        LKMBooks   &250 & \textbf{0.0941} $\pm$ \textbf{0.0004} & 8.28
                   &250 & 0.0780 $\pm$ 0.0073 & 8.47\\
        Raker      &50  & 0.0980 $\pm$ 0.0004 & 8.87
                   &50  & 0.0724 $\pm$ 0.0019 & 8.32\\
        \hline
        RF-IOKS    &400 & 0.0994 $\pm$ 0.0015 &  8.80 &400 & 0.0738 $\pm$ 0.0008 &  8.91\\
        \Xhline{0.8pt}
      \end{tabular}
    \end{table}

\end{document}